\documentclass{article}

\usepackage{amsmath,amssymb,amsthm}
\usepackage{vmargin}

\setpapersize{A4}

\usepackage{graphicx}

\usepackage{url,algorithm,algorithmic}

\newcommand{\g}[1]{\boldsymbol{#1}}

\newcommand{\R}[0]{\mathbb{R}}

\newcommand{\I}[1]{\mathbf{1}_{#1} }	
\newcommand{\Q}[0]{\mathcal{Q}}

\newtheorem{theorem}{Theorem}

\newtheorem{assumption}{Assumption}

\newtheorem{proposition}{Proposition}
\newtheorem{lemma}{Lemma}

\renewenvironment{proof}{{\em Proof.}}{ \hfill\qed\medskip }

\newtheorem{remark}{Remark}
\newcommand{\argmax}{\operatornamewithlimits{argmax}}
\newcommand{\argmin}{\operatornamewithlimits{argmin}}

\begin{document}

\title{\bf Global optimization for low-dimensional switching linear regression and bounded-error estimation}
\author{Fabien Lauer\\\small Universit\'e de Lorraine, LORIA, UMR 7503, F-54506 Vand\oe{}uvre-l\`es-Nancy, France\\\small CNRS}
\date{}

\maketitle

\begin{abstract}
The paper provides global optimization algorithms for two particularly difficult nonconvex problems raised by hybrid system identification: switching linear regression and bounded-error estimation. While most works focus on local optimization heuristics without global optimality guarantees or with guarantees valid only under restrictive conditions, the proposed approach always yields a solution with a certificate of global optimality. This approach relies on a branch-and-bound strategy for which we devise lower bounds that can be efficiently computed. In order to obtain scalable algorithms with respect to the number of data, we directly optimize the model parameters in a continuous optimization setting without involving integer variables. Numerical experiments show that the proposed algorithms offer a higher accuracy than convex relaxations with a reasonable computational burden for hybrid system identification. In addition, we discuss how bounded-error estimation is related to robust estimation in the presence of outliers and exact recovery under sparse noise, for which we also obtain promising numerical results.
\end{abstract}

\section{Introduction}
\label{sec:intro}

The paper tackles two problems that lie at the core of hybrid dynamical system identification, whose aim is to estimate, from input--output data, a model of a system switching at unknown instants between a number of linear subsystems. More precisely, we consider the minimization of the error of a switching linear model with a fixed number of modes and the iterative maximization of the number of data that can be approximated by a linear model with a bounded error. The latter problem, also known as bounded-error estimation, has an interest outside of hybrid systems as well and in particular for robust estimation in the presence of outliers.

The problems are understood as {\em global} minimization/maximization problems. 
However, due to their complexity, most of the literature, as reviewed in \cite{Paoletti07,Garulli12}, focuses on local optimization or heuristic approaches: for switching regression with a fixed number of modes in \cite{Vidal03,Juloski05b,Lauer11a,Lauer13a,Le13c,Lauer14a} and for the bounded-error approach to switching regression in \cite{Bemporad05,Bako11,Ozay12,Diehm13}.
Some of these methods can be proved to yield the global solution but only in specific conditions, such as in the absence of noise for \cite{Vidal03} and under data-dependent conditions difficult to check in practice for \cite{Bako11}. Recent results showed that, though being NP-hard in general, some hybrid system identification problems, including the minimization of the error of a switching linear model, have a complexity no more than polynomial in the number of data for a fixed data dimension \cite{Lauer15a,Lauer15b}. However, in practice, the complexity of the corresponding polynomial algorithms remains too high except for small data sets in small dimensions.

\paragraph*{Contribution}
Global optimization of such difficult problems in general is usually deemed impractical. Hence, we focus on instances where the data can be numerous but should live in a low-dimensional space, as is often the case in a system identification context (most examples in the literature on hybrid system identification have a dimension less than five). 
In this context, the paper proposes a branch-and-bound approach to the two problems above. Contrary to previous works, such an approach offers unconditional global optimality guarantees, while remaining computationally efficient with large data sets. Branch-and-bound is a standard approach to global optimization, but it was only considered for hybrid system identification in \cite{Roll04}, where an off-the-shelf solver is applied after a reformulation of the piecewise affine regression problem into a mixed-integer linear or quadratic program, with a number of binary variables proportional to the number of data. At the opposite, the proposed approach can handle larger data sets by developing dedicated optimization algorithms while focusing on the continuous variables of the problems, i.e., the model parameters, rather than the integer variables. 
Technically, the branch-and-bound approach relies on the derivation of a number of lower bounds on the different cost functions for parameters constrained to lie in a box (a hyperrectangle). In particular, efficiency is obtained thanks to two ingredients: i) simple lower bounds that can quickly discard boxes with very large costs, and ii) a constant-classification based criterion that allows us to more tightly lower bound the cost.

\paragraph*{Paper organization}
Section~\ref{sec:bb} describes the general branch-and-bound approach adopted to tackle the problems of interest, which are formally described in dedicated sections: Sect.~\ref{sec:switched} for switching regression and Sect.~\ref{sec:boundederror} for the bounded-error approach. Then, Sect.~\ref{sec:exp} presents numerical results and Sect.~\ref{sec:conclusion} discusses open issues.

\paragraph*{Notation}
Vectors are written in lowercase bold letters, while matrices are written in uppercase bold letters. For a vector $\g u$, the $k$th entry is denoted by $u_k$, while for a vector $\g u_j$, its $k$th entry is $u_{j,k}$. 
All inequalities between vectors, e.g., $\g u\leq \g v$, are meant entrywise. 
A box $B\subset \R^D$ is a hyperrectangular region of $\R^D$, i.e., $B=[\g u,\g v]=\prod_{k=1}^D [u_k,v_k]$ with $\g u\in\R^D$, $\g v\in\R^D$ such that $\g u\leq \g v$. The positive and negative parts of a scalar are denoted by $(\cdot)_+ = \max\{0,\cdot\}$ and $(\cdot)_- = \min\{0,\cdot\}$ and similar notations are used for the corresponding entrywise operations on vectors. Of course, $(\cdot)_+^2$ and $(\cdot)_-^2$ are understood as the squared positive and negative parts of a scalar, i.e., $(\cdot)_+^2=\left( (\cdot)_+\right)^2$ and $(\cdot)_-^2=\left( (\cdot)_-\right)^2$. The notation $|\cdot|$ denotes either the absolute value for real arguments or the cardinality for sets. The indicator function $\I{A}$ evaluates to $1$ if the boolean expression $A$ is true and $0$ otherwise.

\section{General approach}
\label{sec:bb}

Consider the global minimization of some cost function $J(\g w)$ of a vector of parameters $\g w\in\R^D$ over a box $B_{\rm init}=[\g u_{\rm init},\g v_{\rm init}] \subset \R^D$, where the different definitions of the cost function $J$ for the problems of interest will be given in dedicated sections below. 
We attack these problems with a branch-and-bound approach, summarized in Algorithm~\ref{alg:bb}, which takes a data set of regression vectors $\g x_i\in\R^d$ and target outputs $y_i \in \R$ as inputs. In hybrid system identification, the regression vectors are typically built from lagged inputs and outputs of the system \cite{Paoletti07}.

\begin{algorithm}
\caption{General branch-and-bound scheme. \label{alg:bb}}
\begin{algorithmic}
\REQUIRE A data set $\{(\g x_i,y_i)\}_{i=1}^N \subset \R^d\times \R$, initial box bounds $B_{\rm init}=[\g u_{\rm init}, \g v_{\rm init}]\subset \R^D$ and $TOL > 0$. Optionally, an initial guess of $\g w\in B_{\rm init}$. 
\STATE Initialize the global bounds $\underline{J} \leftarrow 0$, $\overline{J} \leftarrow +\infty$ or $\overline{J} \leftarrow J(\g w)$ if $\g w$ is provided, and the list of boxes $\mathcal{B} \leftarrow \{B_{\rm init}\}$. 
\WHILE{ $(\overline{J}-  \underline{J} )/\overline{J} > TOL$ }
	\STATE {\bf Split} the current box $B$ into $B^1$ and $B^2$ such that $B=B^1\cup B^2$.
	\STATE {\bf Compute upper bounds} $\overline{J}(B^1)$ and $\overline{J}(B^2)$.
	\STATE Update $\overline{J} \leftarrow \min\{\overline{J},\overline{J}(B^1),\overline{J}(B^2)\}$ and the best solution $\g w^*$.
	\STATE {\bf Compute lower bounds} $\underline{J}(B^1)$ and $\underline{J}(B^2)$.
	\STATE For $k=1,2$, append $B^k$ to the list of active boxes $\mathcal{B}$ if $\underline{J}(B^k) \leq \overline{J}$.
	\STATE Remove $B$ from the list of active boxes: $\mathcal{B} \leftarrow \mathcal{B} \setminus \{B\}$.
	\STATE Select the next box $B \leftarrow \argmin_{B\in\mathcal{B}} \underline{J}(B)$ and set $\underline{J} \leftarrow \underline{J}(B)$.
\ENDWHILE
\RETURN $\g w^*$ and $\overline{J} = J(\g w^*) \approx \min_{\g w\in B_{\rm init}} J(\g w)$.
\end{algorithmic}
\end{algorithm}

The general branch-and-bound scheme relies on computing upper and lower bounds ($\overline{J}$ and $\underline{J}$ in Algorithm~\ref{alg:bb}) on the global optimum $\min_{\g w\in B_{\rm init}} J(\g w)$. Then, regions $B$ of the search space in which the local lower bound $\underline{J}(B)$ is larger than the global upper bound $\overline{J}$ can be discarded, reducing the volume left to explore until the relative optimality gap, $(\overline{J}-\underline{J})/\overline{J}$, decreases below a predefined tolerance $TOL$. Here, the considered regions are always boxes, i.e., hyperrectangles. Upper bounds $\overline{J}(B)$ can be easily computed by some local optimization or heuristic method for a problem of interest. Alternatively, $\overline{J}(B)$ can be computed merely as the cost function value at the box base point $\g u$ or at a random point inside the box, while local optimization is only used periodically. On the other hand, lower bounds $\underline{J}(B)$ require a careful derivation, the efficiency of the approach relying mostly on the tightness of these bounds. 

Algorithm~\ref{alg:bb} retains only the solution yielding the best upper bound $\overline{J}=J(\g w^*)$. Depending on the value of $TOL$, the algorithm can terminate while there are multiple remaining active boxes possibly containing equally good solutions within the tolerance. A possible modification would be to retain a list of solution candidates with cost function values close to the best one rather than a single solution. Since such a modification would be straightforward, in the following, we focus only on the version returning a single solution.

\section{Switching linear regression}
\label{sec:switched}

We consider the identification of a switching system with $n$ modes generating a data set of $N$ points $(\g x_i, y_i)\in\R^d\times \R$, $i=1,\dots,N$, with 
\begin{equation}
	y_i = \g w_{q_i}^T \g x_i + \xi_i ,
\end{equation}
where $q_i\in \Q =\{1,\dots,n\}$ is the index of the active mode for the $i$th point, $\{\g w_j\}_{j=1}^n\subset\R^d$ is a collection of linear model parameter vectors and $\xi_i\in\R$ is a noise term. The aim here is to estimate, from the knowledge of $\{(\g x_i,y_i)\}_{i=1}^N$ and $n$ only, the concatenated parameter vector $\g w =  [\g w_1^T,\dots,\g w_n^T]^T \in \R^{nd}$. Throughout the paper, we assume a similar partitioning of all vectors from $\R^{nd}$, i.e., for $\g u\in\R^{nd}$, $\g u_j$ refers to the $j$th subvector of dimension $d$ in $\g u$. 

Least squares estimates\footnote{We restrict the presentation to the squared loss function $\ell(e) = e^2$, but similar results could be obtained for instance with the absolute loss $\ell(e)=|e|$. } of $\g w$ and $\g q=[q_1,\dots,q_N]^T$ are defined as the global solutions to
\begin{align}\label{eq:MIPq}
	&\min_{\g w\in \R^{nd}, \g q \in \Q^N} J_{\rm SWq}(\g w, \g q) ,\\
	&\mbox{with } J_{\rm SWq}(\g w, \g q) = \sum_{i=1}^N  (y_i - \g w_{q_i}^T\g x_i)^2. \nonumber
\end{align}
Note that Problem~\eqref{eq:MIPq} involves $N$ integer variables in $\g q$, which would imply a worst-case exponential complexity in the number of data for its direct global optimization. Other reformulations based on $nN$ binary variables suffer from a similar limitation, which is why the following considers a continuous optimization point of view.

Using the classification rule\footnote{When the minimum is not unique in~\eqref{eq:classif}, ties are arbitrarily broken by returning the minimal index $j$ of the minimum.}
\begin{equation}\label{eq:classif}
	q_i(\g w) = \argmin_{j\in \Q}  (y_i - \g w_j^T \g x_i)^2 ,\quad i=1,\dots, N, 
\end{equation}
Problem \eqref{eq:MIPq} can be reformulated without integer variables as in~\cite{Lauer11a}, leading to
\begin{align}\label{eq:minmin}
	&\min_{\g w \in\R^{nd}} J_{\rm SW}(\g w),\\
	 &\mbox{with } J_{\rm SW}(\g w) = \sum_{i=1}^N \min_{j\in\Q} (y_i - \g w_j^T \g x_i)^2 , \nonumber
\end{align}
and the equivalence $ J_{\rm SW}(\g w) = J_{\rm SWq}(\g w, \g q(\g w) )$. Though equivalent, the formulation in~\eqref{eq:minmin} emphasizes the major role played by the continuous variables in $\g w$, on which directly depends $\g q$. Using this fact, the dimension of the problem can be restricted to $nd$ and becomes independent of $N$.\footnote{Note that the inner minimization over $j$ in~\eqref{eq:minmin} merely amounts to taking the minimum value among $n$ real numbers and should not be seen as an embedded optimization problem.} However, the global optimization of  Problem~\eqref{eq:minmin} remains nontrivial.
  
Note that for symmetry reasons, the cost function $J_{\rm SW}$ is invariant to permutations of the subvectors $\g w_j$ in $\g w$, hence the minimizer is not unique. Such symmetries can be broken by arbitrarily imposing an ordering on the modes, for instance as
\begin{equation}\label{eq:switchedsymmetry}
	w_{j,1} \leq w_{j+1,1} ,\quad j=1,\dots,n-1 ,
\end{equation}
where $w_{j,k}$ denotes the $k$th component of the $j$th parameter vector. 
Note that ties in  the case $w_{j,1} = w_{j+1,1}$ can be broken by imposing similar constraints recursively on the remaining components. However, these additional constraints might be more difficult to take into account in the branch-and-bound approach, while they might also be of little use since the event corresponding to a tie in a global minimizer has zero measure with noisy data. 
Therefore, in the following we do not deal with such ties and focus on solving \eqref{eq:minmin} under the constraints \eqref{eq:switchedsymmetry}. More precisely, we consider the restriction of~\eqref{eq:minmin} subject to \eqref{eq:switchedsymmetry} where the domain $\R^{nd}$ is replaced by a box $B_{\rm init}\subset \R^{nd}$.

\subsection{Branch-and-bound approach}
\label{sec:switchedbb}

Many heuristics have been proposed for switching linear regression (see, e.g., \cite{Garulli12,Lauer11a,Lauer13a,Lauer14a}), and any of them can be used to compute upper bounds $\overline{J}(B)$. Here, we consider the simple and efficient $k$-LinReg method \cite{Lauer13a}. We choose to compute the initial guess of $\g w$ with it before starting the search, and additionally reuse it once in a while (e.g., every 100 iterations). Otherwise, at every iteration dealing with a box $B=[\g u, \g v]$, upper bounds are simply computed as $\overline{J}(B)=J_{\rm SW}(\g u)$.

The symmetry-breaking constraints~\eqref{eq:switchedsymmetry} can be simply imposed at the branching level by explicitly discarding regions of subboxes without feasible solutions. More precisely, we compute $B^1=[\g u^1, \g v^1]$ and $B^2 = [\g u^2, \g v^2]$ from $B=[\g u,\g v]$ by first applying a standard split along the longest side of the box:
\begin{equation}\label{eq:split1}
	(j^*,k^*) = \argmax_{(j,k) \in \Q \times \{1,\dots,d\} } v_{j,k} - u_{j,k}
\end{equation}
and 
\begin{align}\label{eq:split2}
	&\g u^1 = \g u,\quad  v^1_{j,k} = \begin{cases}
							(u_{j,k} + v_{j,k} )/2, & \mbox{if } (j,k)=(j^*,k^*),\\
							v_{j,k}  ,& \mbox{otherwise},
							\end{cases}\\
	& \g v^2 = \g v,\quad u^2_{j,k} = \begin{cases}
				(u_{j,k} + v_{j,k} )/2, & \mbox{if } (j,k)=(j^*,k^*),\\
				u_{j,k}  ,& \mbox{otherwise.}
				\end{cases}
				 \label{eq:split3}
\end{align}
Then, in the case $k^* = 1$, which is the only one concerned by~\eqref{eq:switchedsymmetry}, we correct the box bounds recursively for $j=j^*-1,\dots,1$ with
\begin{equation}\label{eq:split4}
	v^1_{j, 1} = \min \left\{ v^1_{j,1}, v^1_{j+1,1}\right\}
\end{equation}
and, for $j=j^* +1,\dots,n$, with
\begin{equation}\label{eq:split5}
	u^2_{j, 1} = \max \left\{ u^2_{j,1}, u^2_{j-1,1}\right\}.
\end{equation}
Figure~\ref{fig:boxes} illustrates the splitting rule. 
\begin{figure}
\centering
\includegraphics[width=0.5\linewidth]{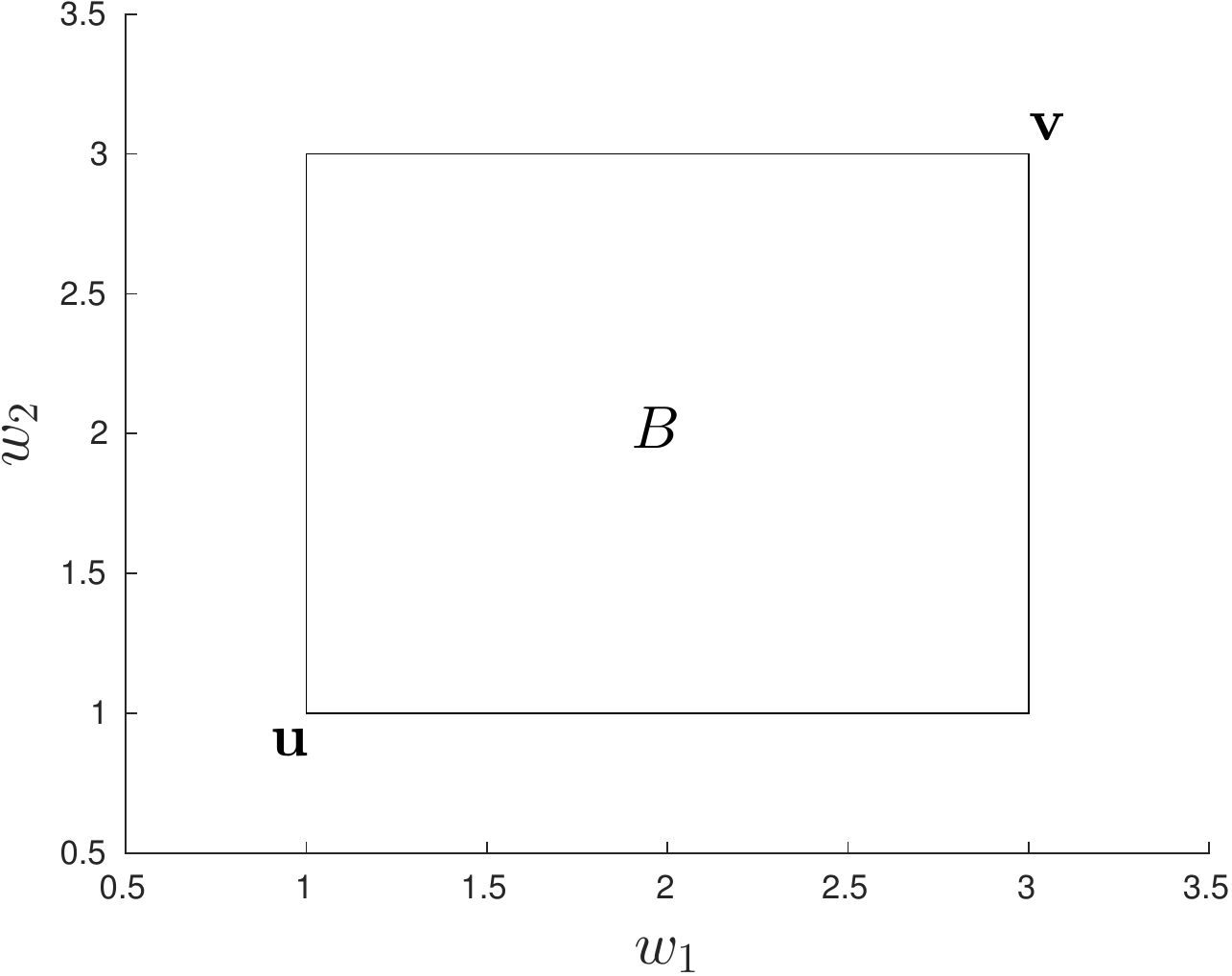}\\
\includegraphics[width=0.5\linewidth]{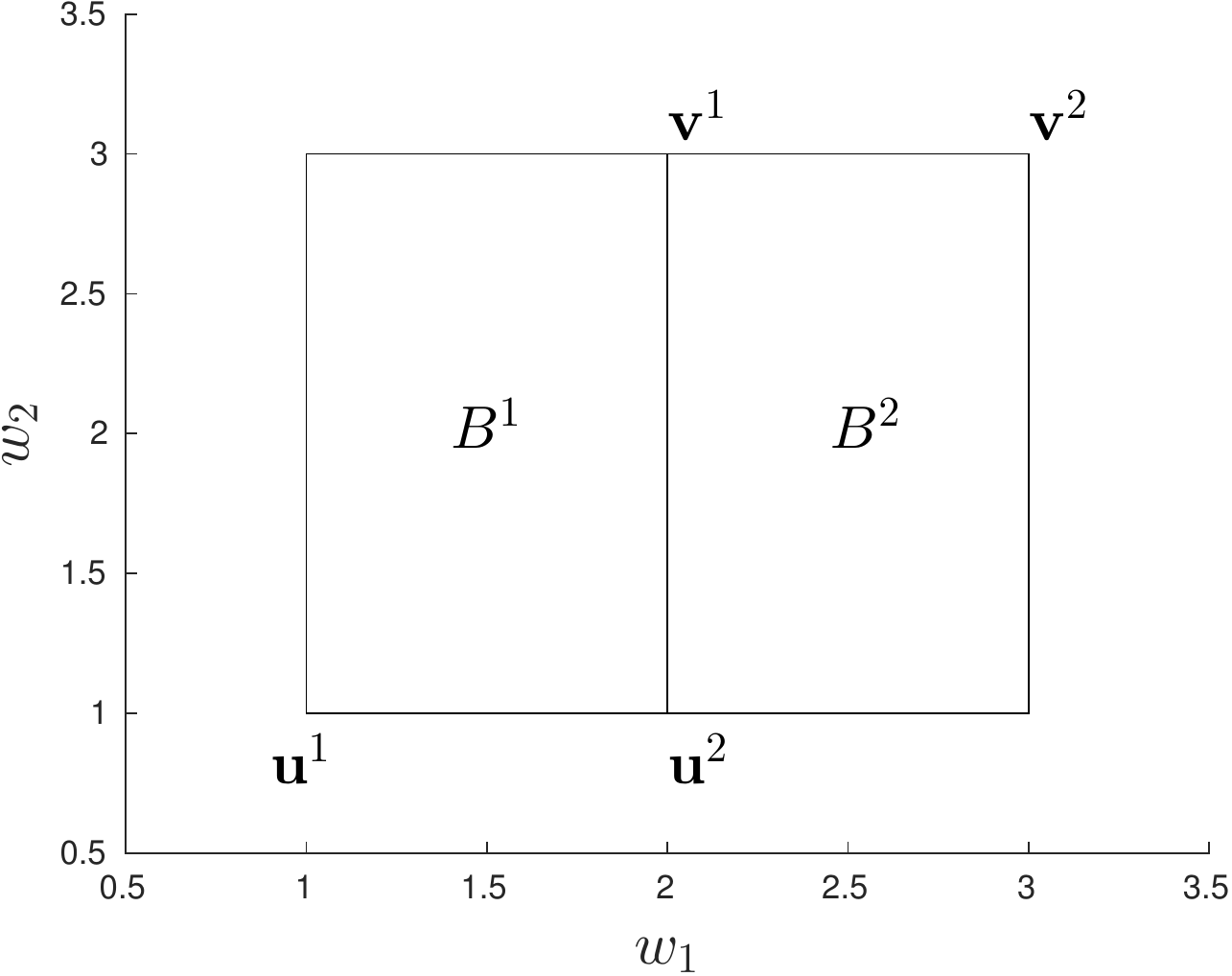}\\
\includegraphics[width=0.5\linewidth]{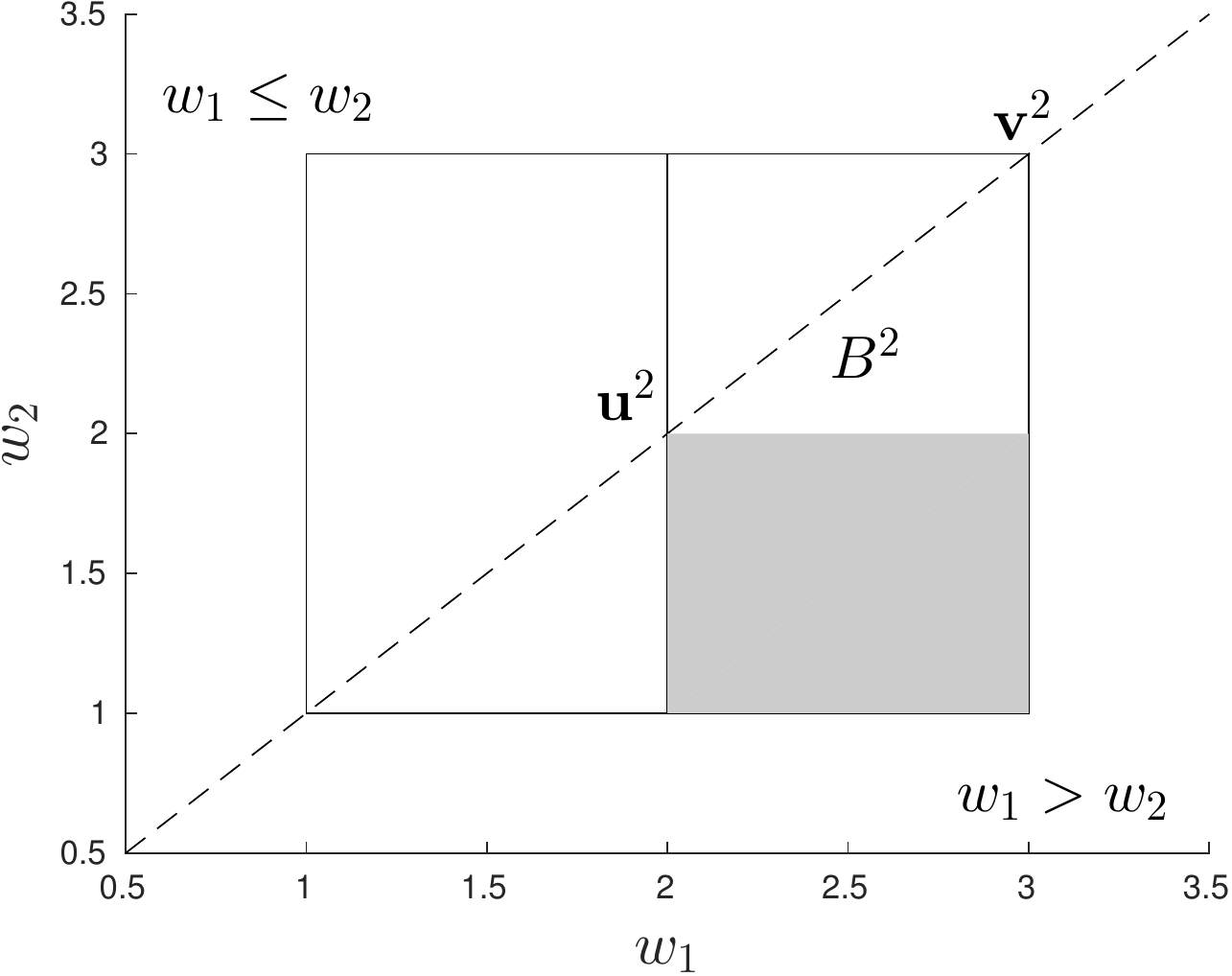}
\caption{Illustration of the splitting procedure when $n=2$ and $d=1$. {\em Top:} a box $B=[\g u,\g v]$ is a rectangular region of the plane of axis $(w_1,w_2)$ with bottom-left and top-right corners at $\g u$ and $\g v$. {\em Middle:} $B$ is split into $B^1=[\g u^1, \g v^1]$ and $B^2=[\g u^2,\g v^2]$ by application of~\eqref{eq:split1}--\eqref{eq:split3}. {\em Bottom: } $B^2$ is corrected as in~\eqref{eq:split4}--\eqref{eq:split5} to remove the shaded area that does not contain any feasible solution according to~\eqref{eq:switchedsymmetry}, i.e., $w_1> w_2$ for all $\g w$ in the shaded area. \label{fig:boxes}}
\end{figure}

\subsection{Lower bounds} 

Regarding the lower bounds, two different bounds are derived below, with increasing value (tightness) but also an increasing computational demand. In practice, we only compute the second lower bound if the first one is not large enough to discard the box (i.e., if it is smaller than the global upper bound $\overline{J}$).

We start with a preliminary result bounding the value of dot products involved in the cost function, which will be of interest throughout the paper.

\begin{lemma}\label{lem:bracketingdotprod}
For any $d$-dimensional box $B_j = [\g u_j, \g v_j]\subset \R^{d}$, we have, for $i=1,\dots,N$,
\begin{equation}\label{eq:dotprod}
	 \begin{cases}
	 	\displaystyle{\min_{\g w_j \in B_j} \g w_j^T \g x_i =  \g u_j^T\g x_i + L_i(B_j) }\\
	 	\displaystyle{\max_{\g w_j \in B_j} \g w_j^T \g x_i =  \g u_j^T\g x_i + U_i(B_j) ,}
	 	\end{cases}
\end{equation}	
where, 
\begin{equation}\label{eq:defLU}
 \begin{cases}
		L_i(B_j) = (\g v_j- \g u_j)^T  (\g x_i)_- \\
		U_i(B_j) =  (\g v_j - \g u_j)^T  (\g x_i)_+ .
	\end{cases}
\end{equation}
\end{lemma}
\begin{proof} 
Any $\g w_j\in [\g u_j, \g v_j]$ can be expressed as
\begin{equation}
	\g w_j = \g u_j + \g \alpha \odot (\g v_j-\g u_j),
\end{equation}
where $\odot$ denotes the entrywise product of vectors and $\g \alpha\in [0,1]^{d}$. Thus, for $i=1,\dots, N$, we have
\begin{equation}
 \begin{cases}
	 	\displaystyle{\min_{\g w_j \in [\g u_j, \g v_j]} \g w_j^T \g x_i =  \g u_j^T\g x_i + \min_{\g\alpha\in[0,1]^d} (\g \alpha \odot (\g v_j-\g u_j) )^T \g x_i}\\
	 	\displaystyle{\max_{\g w_j \in [\g u_j, \g v_j]} \g w_j^T \g x_i =  \g u_j^T\g x_i +\max_{\g\alpha\in[0,1]^d} (\g \alpha \odot (\g v_j-\g u_j) )^T \g x_i} 
	 	\end{cases}
\end{equation}
with  
\begin{align}
	\min_{\g\alpha\in[0,1]^d} (\g \alpha \odot (\g v_j-\g u_j) )^T \g x_i &= \min_{\g\alpha\in[0,1]^d} \g \alpha^T ( (\g v_j-\g u_j) \odot \g x_i ) \nonumber\\
		& = \sum_{k=1}^d \min_{\alpha_k\in[0,1]} \alpha_k (v_{j,k}-u_{j,k}) x_{i,k}\nonumber\\
		& = \g 1^T ( (\g v_j- \g u_j) \odot  \g x_i)_- \nonumber\\
		&= L_i(B_j)
\end{align}
and, similarly, $\max_{\g\alpha\in[0,1]^d} (\g \alpha \odot (\g v_j-\g u_j) )^T \g x_i = U_i(B_j)$.
\end{proof}

\subsubsection{Lower bound based on pointwise minimum errors}

The first lower bound is based on a pointwise decomposition of the optimization problem with respect to the index $i$ of data points. In particular, we use the fact that, for any $i\in\{1,\dots,N\}$ and $j\in\Q$, the pointwise error of a parameter vector $\g w_j$ at a given point $(\g x_i,y_i)$, 
\begin{equation}
	e_i(\g w_j) = y_i - \g w_j^T \g x_i, 
\end{equation}
can be made smaller in magnitude if we are not trying to simultaneously minimize the errors at other points.
Therefore, the global cost $J_{\rm SW}(\g w)$ must be at least as large as the sum of independently optimized pointwise errors. Formally, for any $i\in\{1,\dots,N\}$ and box $B_j=[\g u_j, \g v_j]\subset \R^d$, let 
\begin{align}\label{eq:defeiLU}
		e_i^L(B_j) = e_i(\g u_j) - L_i(B_j)\\
		e_i^U(B_j) = e_i(\g u_j) - U_i(B_j)\nonumber
\end{align}
with $L_i(B_j)$ and $U_i(B_j)$ as in~\eqref{eq:defLU}. Then, 
we have the following lower bound.
\begin{lemma}\label{lem:pwe}
Given a box $B= B_1\times\dots\times B_n$ with $B_j=[\g u_j, \g v_j] \subset \R^{d}$, $j=1,\dots,n$, and the notations defined above, 
\begin{equation}
	\underline{J}(B) = \sum_{i=1}^N \min_{j\in\Q} \left\{\left(e_i^U(B_j) \right)_+^2  + \left(e_i^L(B_j)\right)_-^2 \right\}
\end{equation}
is a lower bound on $\min_{\g w \in B} J_{\rm SW}(\g w)$.
\end{lemma}
\begin{proof}
For any $\g a\in B$, 
\begin{equation}
	\min_{j\in\Q} e_i^2(\g a_j) \geq  \min_{\g w \in B} \min_{j\in\Q} e_i^2(\g w_j) = \min_{j\in\Q}  \min_{\g w \in B}  e_i^2(\g w_j) 
\end{equation}
and, by summing over $i$, 
\begin{equation}
	J_{\rm SW}(\g a)\geq \sum_{i=1}^N\min_{j\in\Q}  \min_{\g w \in B}  e_i^2(\g w_j) = \sum_{i=1}^N\min_{j\in\Q}  \min_{\g w_j \in B_j}  e_i^2(\g w_j) .
\end{equation}
Since this holds for any $\g a\in B$, it holds in particular for the one yielding the minimum of $J_{\rm SW}$ over $B$ and we obtain
\begin{equation}
\min_{\g w \in B} J_{\rm SW}(\g w) \geq \sum_{i=1}^N \min_{j\in\Q} \min_{\g w_j \in B_j} e_i^2(\g w_j) . \label{eq:lbpwe1}
\end{equation}
On the other hand, Lemma~\ref{lem:bracketingdotprod} yields
\begin{equation}\label{eq:bracketingeij}
	\begin{cases}
	 	\displaystyle{\min_{\g w_j \in B_j} e_i(\g w_j) =  e_i(\g u_j) - U_i(B_j) = e_i^U(B_j)} \\
	 	\displaystyle{\max_{\g w_j \in B_j} e_i(\g w_j) =  e_i(\g u_j) - L_i(B_j) = e_i^L(B_j),}
	 	\end{cases}
\end{equation}
and thus
\begin{equation}
	\min_{\g w_j \in B_j} |e_i(\g w_j)| = \begin{cases}
		e_i^U(B_j),& \mbox{if } e_i^U(B_j) > 0\\
		|e_i^L(B_j)|,& \mbox{if } e_i^L(B_j) < 0\\		
		0,& \mbox{otherwise },
	\end{cases}
\end{equation}
which can be rewritten as 
\begin{equation}\label{eq:mineij}
	\min_{\g w_j \in B_j} e_i^2(\g w_j) = \left(e_i^U(B_j)\right)_+^2  + \left(e_i^L(B_j)\right)_-^2.
\end{equation}
Combining~\eqref{eq:lbpwe1} and~\eqref{eq:mineij} completes the proof. 
\end{proof}

From a computing time perspective, once an upper bound has been computed as $\overline{J}(B) = J_{\rm SW}(\g u)$, only the values of $L_i(B_j)$ and $U_i(B_j)$ are needed to compute $e_i^L(B_j)$ and $e_i^U(B_j)$ for the evaluation of the lower bound $\underline{J}(B)$ in Lemma~\ref{lem:pwe}; and, as will be seen below, these quantities will be used again.

\subsubsection{Lower bound based on constant classifications}

Let us denote by $I=\{1,\dots,N\}$ the set of all point indexes and, for all $j\in\Q$, the subset of indexes for which the classification remains constant and equal to $j$ in a box $B=[\g u, \g v]\subset \R^{nd}$ by
\begin{equation}\label{eq:Ij}
	I_j(B) = \{ i\in I : \forall \g w\in B,\ q_i(\g w) =j \} ,
\end{equation}
where $q_i(\g w)$ is given by \eqref{eq:classif}. 
We also define $I_0(B)=I \setminus \cup_{j=1}^n I_j(B)$ as the subset of the remaining indexes. Thus, any data index $i\in I$ is exactly in one and only one of the sets $I_j(B)$, $j=0,\dots,n$. 
The following result shows how these index sets can be determined.

\begin{lemma}\label{lem:constantclassif}
Given a box $B= B_1\times\dots\times B_n$ with $B_j=[\g u_j, \g v_j] \subset \R^{d}$, $j=1,\dots,n$, and the notations above, we have, for $j=1,\dots,n$, 
\begin{align}
	I_j (B)=  \Big\{ i\in I: 
	&\max \left\{ e_i^U(B_j)^2 , e_i^L(B_j)^2 \right\} <  \min_{k< j} \left(e_i^U(B_k)\right)_+^2+\left(e_i^L(B_k)\right)_-^2 ,\nonumber\\
	&\max \left\{ e_i^U(B_j)^2 , e_i^L(B_j)^2 \right\} 
	 \leq  \min_{k> j} \left(e_i^U(B_k)\right)_+^2+\left(e_i^L(B_k)\right)_-^2\Big\}.\nonumber
\end{align}
\end{lemma}
\begin{proof}
For any $j\in\Q$, given the definition of $q_i$ in~\eqref{eq:classif} and the fact that ties are broken by setting $q_i(\g w)$ as the smallest mode index, the set $I_j(B)$ contains exactly the indexes $i\in I$ for which 
\begin{align}
	\forall \g w\in B,&\  \forall k\in\Q\setminus\{j\}, \quad (y_i - \g w_j^T \g x_i)^2 
	\begin{cases}
		< (y_i - \g w_k^T \g x_i)^2 ,&\mbox{if } k<j\\
		\leq (y_i - \g w_k^T \g x_i)^2 ,&\mbox{if } k>j.
	\end{cases}\nonumber
\end{align}
Since the left-hand side and the right-hand side depend on different components, $B_j$ and $B_k$, of $B$ and since with the constraint $\g w\in B$ the component $\g w_j\in B_j$ can freely evolve independently of $\g w_k\in B_k$, this condition is equivalent to
\begin{align}
	 \forall \g w_j\in B_j,&\ \forall k\in\Q\setminus\{j\},\ \forall \g w_k\in B_k,\quad (y_i - \g w_j^T \g x_i)^2 
	\begin{cases}
		< (y_i - \g w_k^T \g x_i)^2 ,&\mbox{if } k<j\\
		\leq (y_i - \g w_k^T \g x_i)^2 ,&\mbox{if } k>j,
	\end{cases}\nonumber
\end{align}
which can be rewritten in the more compact form
\begin{equation}
	 \begin{cases}
	 	\displaystyle{\max_{\g w_j \in B_j} e_i^2(\g w_j) < \min_{k< j} \min_{\g w_k \in B_k} e_i^2(\g w_k) }\\
	 	\displaystyle{\max_{\g w_j \in B_j} e_i^2(\g w_j) \leq \min_{k> j} \min_{\g w_k \in B_k} e_i^2(\g w_k) } .
	 \end{cases}
\end{equation}
Using~\eqref{eq:mineij} with $k$ instead of $j$ in the right-hand sides and noting that~\eqref{eq:bracketingeij} yields the left-hand sides as
\begin{equation}\label{eq:maxeij}
	\max_{\g w_j \in B_j} e_i^2(\g w_j) = \max\left\{e_i^U(B_j)^2 , e_i^L(B_j)^2 \right\}
\end{equation}
completes the proof.
\end{proof}

Once the sets $I_j(B)$, $j=0,\dots,n$, have been determined, the following gives an improved lower bound by constraining the error over points with index in a set $I_j(B)$ for $j\geq 1$ to be computed with respect to a single linear model. 
\begin{lemma}\label{lem:lbclassif}
Given a box $B=[\g u, \g v]\subset \R^{nd}$ and the notations above, for any $\mathcal{J} \subseteq \Q$,
\begin{align}
	\underline{J}(B)=&\sum_{i\in I_0(B)}\min_{j\in\Q} \left\{\left(e_i^U(B_j)\right)_+^2  + \left(e_i^L(B_j)\right)_-^2 \right\} 
	+\sum_{j \in \mathcal{J}}\min_{\g w_j \in B_j} \sum_{i\in I_j(B)} (y_i - \g w_j^T \g x_i)^2 \label{eq:lbclassif}\\
	&+\sum_{j \in \Q\setminus \mathcal{J}}\sum_{i\in I_j(B)}\left(e_i^U(B_j)\right)_+^2+\left(e_i^L(B_j)\right)_-^2 \nonumber
\end{align}
is a lower bound on $\min_{\g w \in B} J_{\rm SW}(\g w)$.
\end{lemma}
\begin{proof}
Since each data point index $i\in I$ is exactly in one and only one of the sets $I_j(B)$, $j=0,\dots,n$, 
for all $\g w\in B$, the cost function in~\eqref{eq:minmin} can be rewritten as 
\begin{align}
 J_{\rm SW}(\g w) &= \sum_{j=0}^n \sum_{i\in I_j(B)} \min_{j\in\Q} (y_i - \g w_j^T\g x_i)^2 \\
 		&=  \sum_{i\in I_0(B)}\min_{j\in\Q} e_i^2(\g w_j) + \sum_{j=1}^n \sum_{i\in I_j(B)} e_i^2(\g w_j)  ,\nonumber 		
\end{align}
where the second line is due to the definition of the $I_j(B)$'s in~\eqref{eq:Ij}. Then, the sum over $I_0(B)$ can be lower bounded as in~\eqref{eq:lbpwe1}, while each of the other terms only involves the $j$th subvector $\g w_j$:
\begin{align}
	\min_{\g w \in B} J_{\rm SW}(\g w)  \geq \sum_{i\in I_0(B)} \min_{j\in\Q} \min_{\g w_j \in B_j} e_i^2(\g w_j) + \sum_{j=1}^n  \min_{\g w_j \in B_j} \sum_{i\in I_j(B)} e_i^2(\g w_j).
\end{align}
Substituting~\eqref{eq:mineij} in the above for all $i\in I_0(B)$ proves the Lemma for the case $\mathcal{J} = \Q$. To complete the proof for any $\mathcal{J}\subset\Q$, we write
\begin{align}
	\sum_{j=1}^n  \min_{\g w_j \in B_j} \sum_{i\in I_j(B)} e_i^2(\g w_j) \geq &\sum_{j\in\mathcal{J}}  \min_{\g w_j \in B_j} \sum_{i\in I_j(B)} e_i^2(\g w_j)  + \sum_{j\in\Q\setminus\mathcal{J}} \sum_{i\in I_j(B)}  \min_{\g w_j \in B_j} e_i^2(\g w_j)
\end{align}
and lower bound the second term by invoking again~\eqref{eq:mineij} for all $i\in I_j(B)$ and $j\in \Q\setminus \mathcal{J}$.
\end{proof}

Lemma~\ref{lem:lbclassif} can lower bound the error over the points with index in $I_j(B)$ for $j\geq 1$ in two different manners: the second term in~\eqref{eq:lbclassif} corresponds to a box-constrained least squares error, while the third term is the sum of pointwise minimum errors for the $j$th model. Thus, the degree of freedom left by the choice of $\mathcal{J}$ in Lemma~\ref{lem:lbclassif} can be used to trade off computing time for accuracy: solving a constrained least squares problem yields a larger lower bound but is more demanding than summing pointwise minimum errors (the latter being computed with very little effort given that all quantities involved have already been computed when determining $I_j(B)$). 
Thus, Lemma~\ref{lem:lbclassif} can be used to provide a sequence of lower bounds with increasing complexity: starting with $\mathcal{J}=\emptyset$, we increment the cardinality of $\mathcal{J}$ and stop as soon as the lower bound reaches the upper bound, at which point the box can be discarded from the search.

\subsection{Convergence} 
\label{sec:switchingcvg}

The convergence of the branch-and-bound algorithm follows from the tightness of the lower bound in Lemma~\ref{lem:lbclassif}. It is formally stated below under two assumptions.
\begin{assumption}\label{ass:globalpositive}
The global optimum of Problem~\eqref{eq:minmin} is strictly positive:
\begin{equation}
	J^* = \min_{\g w\in B_{\rm init}} J_{\rm SW}(\g w)	> 0.
\end{equation}
\end{assumption}
\begin{assumption}\label{ass:upperbounds}
Upper bounds in a box $B=[\g u,\g v]$ are computed as $\overline{J}(B) = J_{\rm SW}(\g u)$ or such that $\overline{J}(B) \leq J_{\rm SW}(\g u)$. 
\end{assumption}
Assumption~\ref{ass:globalpositive} merely requires that the data cannot be exactly fitted by $n$ linear models, which is most often the case with noisy measurements. Otherwise, the convergence of Algorithm~\ref{alg:bb} can be proved similarly for noiseless data by changing the stopping criterion on the relative accuracy, $(\overline{J}-  \underline{J} )/\overline{J} \leq TOL$, for one on the absolute accuracy: $\overline{J}-  \underline{J} \leq TOL$.
Assumption~\ref{ass:upperbounds} simply requires that the upper bounds are at least as accurate as the straightforward computation of the cost function value at the box base point $\g u$. Though this depends on the precise choice of heuristic for computing the upper bounds, Assumption~\ref{ass:upperbounds} can always be made to hold simply by setting $\overline{J}(B) \leftarrow \min\{\overline{J}(B) , J_{\rm SW}(\g u)\}$. 
\begin{theorem}\label{thm:converge}
Under Assumptions~\ref{ass:globalpositive}--\ref{ass:upperbounds}, Algorithm~\ref{alg:bb} as described above with 
lower bounds $\underline{J}(B)$ computed as in Lemma~\ref{lem:lbclassif} with $\mathcal{J}=\Q$ converges in a finite number of iterations for any $TOL>0$.
\end{theorem}
\begin{proof}
Let $t$ denote the iteration counter in Algorithm~\ref{alg:bb}. 
The splitting rule detailed in Sect.~\ref{sec:switchedbb} guarantees that the side lengths of the boxes decrease towards 0, i.e., $\g v - \g u\xrightarrow{t\to\infty} \g 0$ and $B\xrightarrow{t\to\infty}\{\g u\}$, for all active boxes $B$ with lower bound $\underline{J}(B)$ smaller than the upper bound $\overline{J}$. Then, by their definitions in~\eqref{eq:defLU}, this implies the convergence of $L_i(B_j)$ and $U_i(B_j)$ towards zero for all data index $i\in I$ and mode $j\in\Q$. Recalling~\eqref{eq:defeiLU}, this leads to $e_i^L(B_j)\xrightarrow{t\to\infty} e_i(\g u_j)$ and $e_i^U(B_j)\xrightarrow{t\to\infty} e_i(\g u_j)$. Hence, by the definition of $\underline{J}(B)$ in Lemma~\ref{lem:lbclassif}, we obtain for any active $B$ that
\begin{align}
	\underline{J}(B) \xrightarrow{t\to\infty} &\sum_{i\in I_0(B)} \min_{j\in\Q}  e^2_i(\g u_j)   + \sum_{j \in \Q} \min_{\g w_j \in B_j} \sum_{i\in I_j(B)} e^2_i(\g w_j)\nonumber\\
	\xrightarrow{t\to\infty} &  \sum_{i\in I_0(B)} \min_{j\in\Q} e^2_i(\g u_j) + \sum_{j \in \Q} \sum_{i\in I_j(B)} e^2_i(\g u_j).
\end{align}
By the definition of the index sets $I_j(B)$, $j=0,\dots,n$, and since each data index $i\in I$ is exactly in one and only one of these sets, the right-hand side above equals $\sum_{i\in I} \min_{j\in\Q} e^2_i(\g u_j)$ and we have
\begin{equation}
	\underline{J}(B) \xrightarrow{t\to\infty} \ J_{\rm SW}(\g u) . 
\end{equation}
Thus, for all $\epsilon>0$, there is a finite iteration number $T$ such that for all subsequent iterations $t\geq T$ and all boxes $B$ in the list of active boxes $\mathcal{B}$, 
\begin{equation}
	\epsilon \geq J_{\rm SW}(\g u) - \underline{J}(B) \geq  \overline{J}(B) - \underline{J}(B) ,
\end{equation}
where the second inequality is due to Assumption~\ref{ass:upperbounds}. 
In particular, consider some $\epsilon \in (0, TOL \cdot J^*]$ (such an $\epsilon$ exists by Assumption~\ref{ass:globalpositive}) and let $B$ be the current box at iteration $T$. Recall that Algorithm~\ref{alg:bb} is such that, at this iteration, $\overline{J}\leq \overline{J}(B)$ and $\underline{J} = \underline{J}(B)$. Then, the fact that $\overline{J}$ is an upper bound on $J^*$ ensures that the stopping criterion is met:
\begin{equation}
	\frac{\overline{J} - \underline{J} }{\overline{J} } \leq \frac{\overline{J}(B) - \underline{J}(B) }{ \overline{J} } \leq \frac{\epsilon }{\overline{J}} \leq \frac{ TOL \cdot J^*}{ \overline{J}} \leq TOL ,
\end{equation} 
and the algorithm stops at iteration $T$.
\end{proof}

\begin{remark}
The worst-case time complexity of the branch-and-bound approach described above remains exponential in the number of variables, $nd$, as is naturally expected from the NP-hardness of Problem~\eqref{eq:minmin} \cite{Lauer15b}. However, in practice, the average time complexity can be much lower thanks to the rapid contraction of the search space offered by the bounding scheme. In the product $nd$, the main parameter influencing the computing time is the number of modes, $n$. Indeed, increasing $n$ does not only increase the number of splits required to obtain small boxes with small values of $L_i(B_j)$ and $U_i(B_j)$ as for $d$, but also affects the quality of the lower bounds involving the operation $\min_{j\in\Q}$. Therefore, for large $n$, the iterative approach described in the next section can be more suitable.
\end{remark}

\section{Bounded-error estimation} 
\label{sec:boundederror}

Consider now the problem of estimating both the minimal number $n$ and the collection of $n$ parameter vectors $\{\g w_j\}_{j=1}^n$ such that, given a threshold $\epsilon$ on the error, 
\begin{equation}\label{eq:boundederror}
	\min_{j\in\{1,\dots,n\}} | y_i - \g w_j^T \g x_i| \leq \epsilon, \quad i=1,\dots, N.
\end{equation}
This setting has a long history in the hybrid system identification literature \cite{Bemporad05,Bako11}. 
In this framework, the number of models is not fixed a priori but estimated in order to satisfy the bound on the error. It is shown in \cite{Amaldi02} that this problem is NP-hard and standard approaches work in a greedy manner by estimating one model at each iteration until the bounded-error constraint is satisfied for all points. 
After estimating the $j$th model, the data correctly approximated in the sense of \eqref{eq:boundederror} are removed from the data set before proceeding with the next iteration. 
In \cite{Bemporad05}, the $j$th model is defined as the one correctly approximating the maximum number of points and is estimated by finding the feasible subsystem of inequalities~\eqref{eq:boundederror} of maximum cardinality (the MAX FS problem). However, finding such a subsystem is itself shown to be NP-hard in \cite{Amaldi95} and \cite{Bemporad05} has to rely on a suboptimal heuristic. 

Focusing on the noiseless case, \cite{Bako11} poses this problem as a sparse optimization one: the model yielding a perfect fit of the maximum number of points is the one that yields the sparsest error vector $\g e_j = \g y_j - \g X_j \g w_j$ (with $\g X_j$ and $\g y_j$ containing the remaining data at the $j$th iteration). Though sparse optimization problems are NP-hard in general \cite{Natarajan95}, the advantage of this point of view is that guarantees of convergence of heuristics (such as the $\ell_1$-norm based convex relaxation) to the exact solution can be obtained \cite{Foucart13}. However, in the noisy case, $\g e_j$ is not sparse anymore but ``compressible" and the guarantees become weaker. Then, the method of \cite{Bako11} basically behaves as a robust estimator 
based on the $\ell_1$-loss when considering points from the non-dominant mode as outliers \cite{Bako16}. 

The two points of view from \cite{Bemporad05} and \cite{Bako11} can be united by considering that a vector is sparse when a large fraction of its entries are not larger than a given threshold~$\epsilon$. This leads to the minimization of the loss function $\ell_{0,\epsilon}(e) = \I{|e| > \epsilon}$, plotted on the left of Fig.~\ref{fig:losses}, instead of the $\ell_0$-pseudo norm of $\g e_j$. This minimization is directly equivalent to the MAX FS problem of \cite{Bemporad05}, while the setting of \cite{Bako11} is recovered with $\epsilon=0$. 
By generalizing these ideas, we also consider the minimization of saturated loss functions, 
\begin{equation}
	\forall p\in\{1,2\}, \epsilon>0,\quad \ell_{p,\epsilon}(e) = (\min \{|e|,\ \epsilon \} )^p ,
\end{equation}
which yield a nonzero loss for $0<e<\epsilon$ (see the middle and right plots of Fig.~\ref{fig:losses}). 
Indeed, depending on the noise model, a better alternative might be to not only obtain error vectors with many small entries, but also to minimize a standard loss over these small entries, such as a squared error or absolute deviation. 

\begin{figure}
\centering
\includegraphics[width=0.32\linewidth]{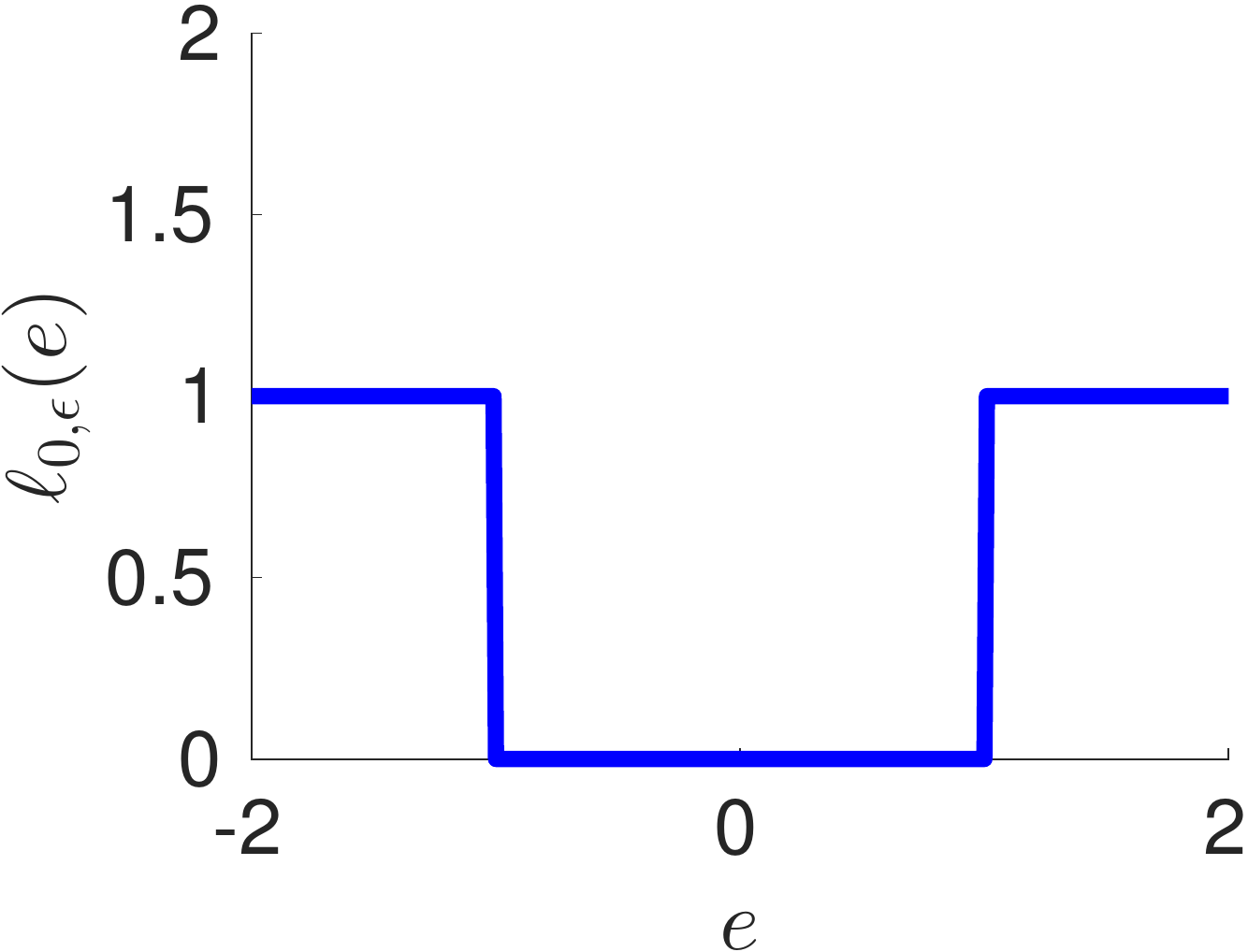}
\includegraphics[width=0.32\linewidth]{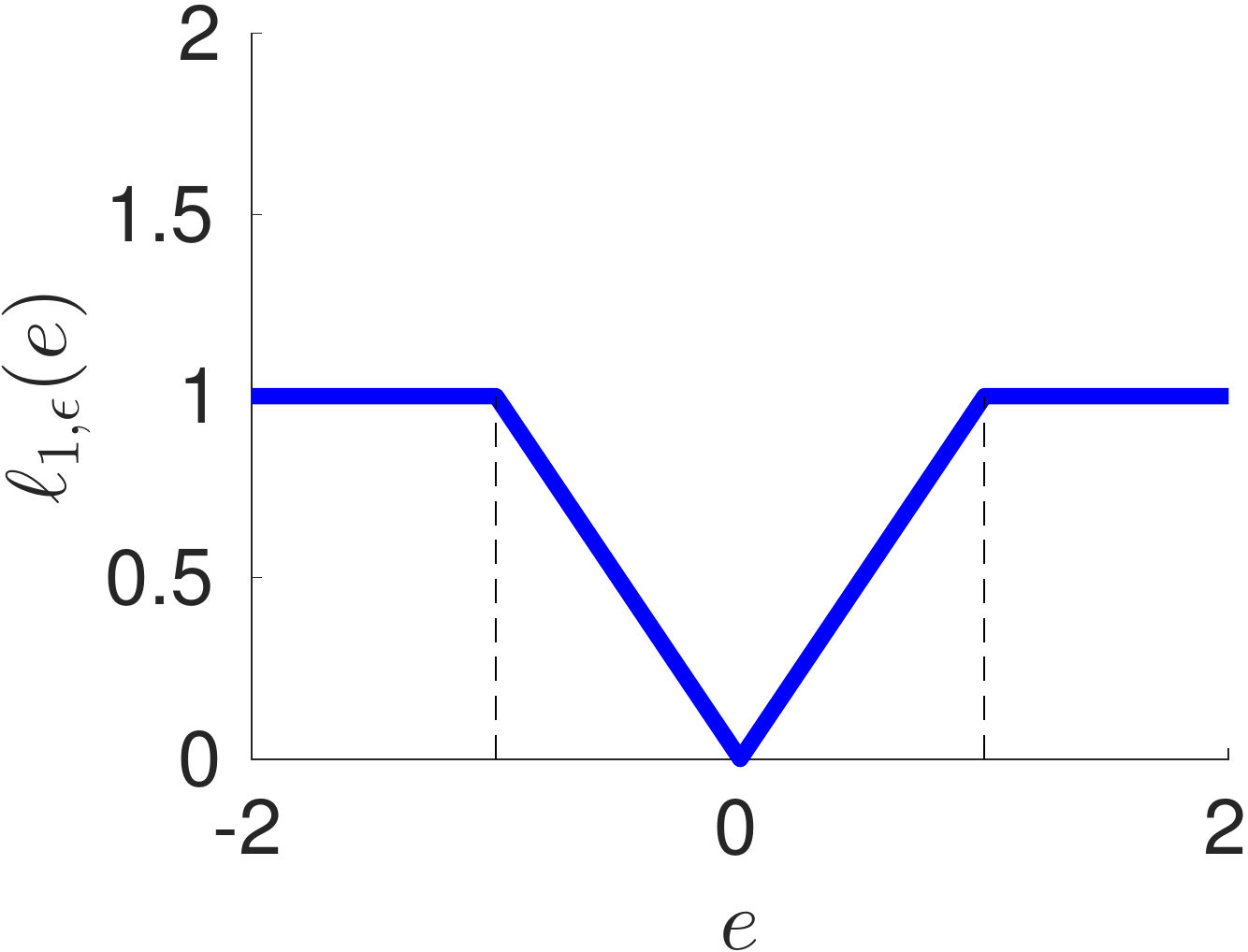}
\includegraphics[width=0.32\linewidth]{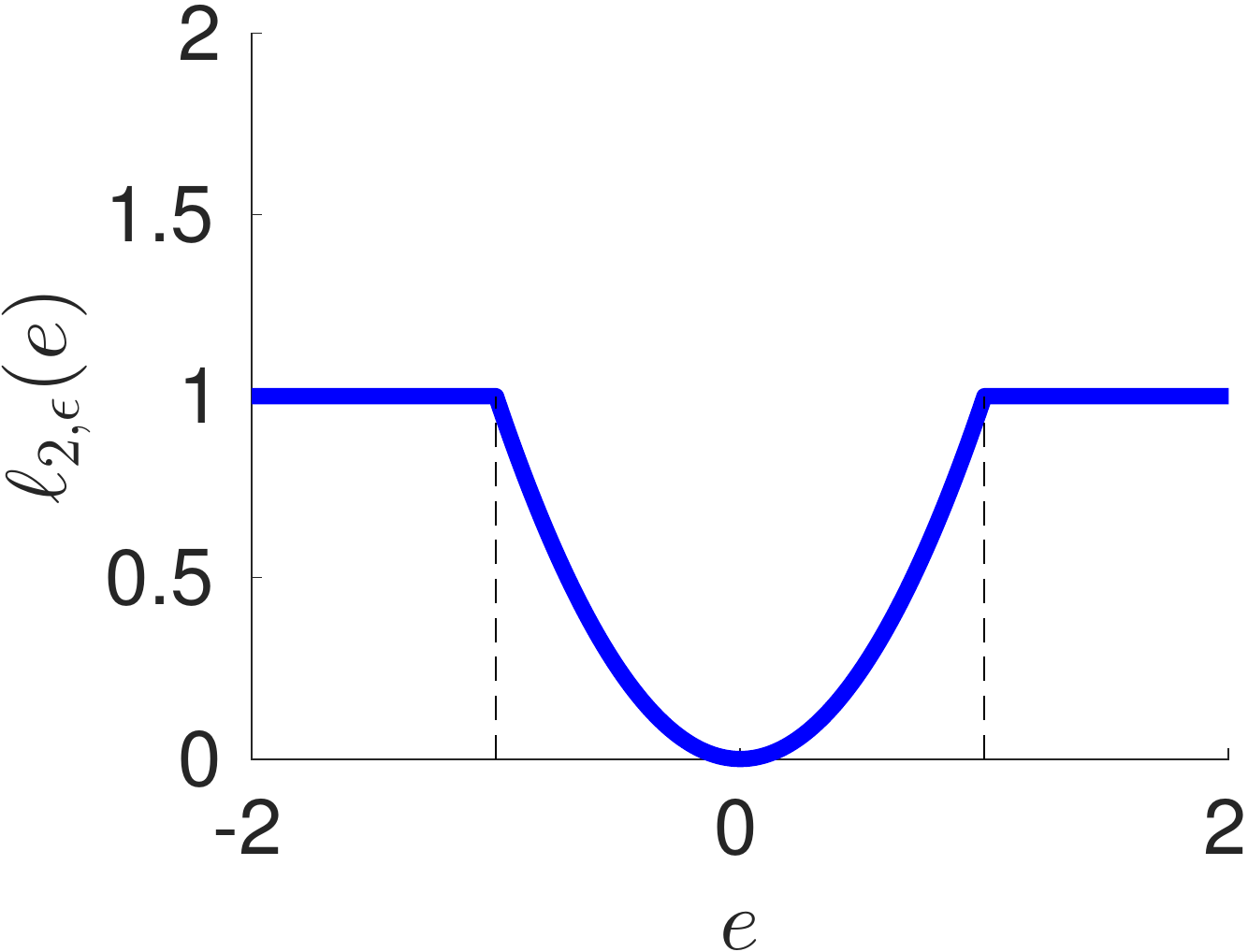}
\caption{Saturated loss functions $\ell_{p,\epsilon}(e)$ plotted for $\epsilon=1$ with, from left to right, $p=0$, $p=1$ and $p=2$.\label{fig:losses}}
\end{figure}

Overall, the considered bounded-error approach relies on iteratively solving for $j=1,2,\dots$ the nonconvex optimization problem
\begin{align}\label{eq:Jbe}
	&\min_{\g w_j \in B\subset \R^d} J_{\rm BE}^p(\g w_j),\\
	&\mbox{with } J_{\rm BE}^p(\g w_j)  = \sum_{i \in I_j} \min \left\{ |y_i - \g w_j^T \g x_i|^p ,\ \epsilon^p \right\}\nonumber
\end{align}
or
\begin{align}\label{eq:Jbe0}
	&\min_{\g w_j \in B\subset \R^d} J_{\rm BE}^0(\g w_j),\\
	&\mbox{with } J_{\rm BE}^0(\g w_j)  = \sum_{i\in I_j} \I{|y_i - \g w_j^T \g x_i| > \epsilon} ,\nonumber
\end{align}
where $I_1 = \{1,\dots,N\}$ and $I_j=\{ i\in I_{j-1} : |y_i - \g w_j^T \g x_i| > \epsilon \} $, until the set of indexes $I_j$ of points left in the data set and unassigned to a mode becomes empty.  

Note that these problems are also of interest for the robust estimation of a single (non-hybrid) linear model in the presence of outliers, as will be considered in the examples of Sect.~\ref{sec:exprobust}--\ref{sec:expsparse} (see also \cite{Bako16} for an analysis of the sparse optimization method of \cite{Bako11} in this context). 

\subsection{Branch-and-bound approach}

We now focus on the minimization of $J_{\rm BE}^p$ in the cases $p=2$ and $p=0$. 
In devising a branch-and-bound algorithm to solve \eqref{eq:Jbe} and~\eqref{eq:Jbe0}, we will use techniques developed in Section~\ref{sec:switched} for switching regression. Indeed, the bounded-error cost functions are closely related to $J_{\rm SW}$ and their minimization can be seen as a switching regression problem with two modes, one of which having a model with constant error equal to $\epsilon^2$ for $J_{\rm BE}^2$ or $1$ for $J_{\rm BE}^0$.

Upper bounds on the bounded-error cost functions are simply computed as the cost function value at the base point $\g u$ of the current box $B=[\g u,\g v]\subset\R^d$. In the case $p=2$, we also periodically use the heuristic described below with an initialization at the center of the box. 

Since in bounded-error estimation, each optimization problem, either~\eqref{eq:Jbe} or~\eqref{eq:Jbe0}, aims at the estimation of a single parameter vector, there is no symmetry that needs to be broken with additional constraints as we did for switching regression with~\eqref{eq:switchedsymmetry}. Therefore, the splitting rule remains simple: we split in the middle of the longest side of the box with~\eqref{eq:split1}--\eqref{eq:split3}. 

\subsection{Heuristic for bounded-error estimation}

We introduce a new heuristic inspired by the $k$-LinReg algorithm \cite{Lauer13a} to minimize the bounded-error cost with the $\ell_{2,\epsilon}$ loss. Starting at iteration $t=0$ from a parameter vector $\g w^0$ and a data set $\{(\g x_i, y_i)\}_{i\in I_j}$, the algorithm alternates between a classification step and a least squares regression as follows.
\begin{enumerate}
	\item Identify the set of points satisfying the bounded-error criterion:
	\begin{equation}
		I_B(t) = \left\{i\in I_j : (y_i - \g x_i^T \g w_j^t)^2 \leq \epsilon^2 \right\} .
	\end{equation}
	\item Update the parameter vector with 
	\begin{equation}\label{eq:boudnederrorheuristic}
		\g w_j^{t+1} = \argmin_{\g w\in\R^d} \sum_{i\in I_B(t)} (y_i- \g w^T \g x_i)^2.
	\end{equation}
	\item Set $t\leftarrow t+1$ and repeat from step (1) until convergence. 
\end{enumerate}

The following shows that this algorithm is a descent method for problem~\eqref{eq:Jbe} with $p=2$.
\begin{proposition}
The algorithm above monotonically decreases the bounded-error cost function with $\ell_{2,\epsilon}$ loss, $J_{\rm BE}^2$. 
\end{proposition}
\begin{proof}
Partitioning the set of all indexes into three disjoint subsets as
\begin{equation}
	I_j = I_B(t) \cup \left(I_B(t+1) \setminus I_B(t)\right) \cup \left( I_j \setminus(I_B(t)\cup I_B(t+1)\right),
\end{equation}
we write the difference between the cost function values at the new estimate, $\g w_j^{t+1}$, and at the old one, $\g w_j^t$, as
\begin{align}
	J_{\rm BE}^2(\g w_j^{t+1}) - J_{\rm BE}^2(\g w_j^t) 
	=& \sum_{i \in I_B(t) } \ell_{2,\epsilon} ( y_i -  \g x_i^T \g w_j^{t+1} ) - \ell_{2,\epsilon} ( y_i -  \g x_i^T \g w_j^{t} ) \nonumber\\
	& + \sum_{i \in I_B(t+1) \setminus I_B(t) } \ell_{2,\epsilon} ( y_i -  \g x_i^T \g w_j^{t+1} ) - \ell_{2,\epsilon} ( y_i -  \g x_i^T \g w_j^{t} ) \nonumber\\
	&+ \sum_{ I_j \setminus(I_B(t) \cup I_B(t+1) ) }\ell_{2,\epsilon} ( y_i -  \g x_i^T \g w_j^{t+1} ) - \ell_{2,\epsilon} ( y_i -  \g x_i^T \g w_j^{t} ).\nonumber
\end{align}
The first sum is negative since the new estimate $\g w_j^{t+1}$ is precisely defined in~\eqref{eq:boudnederrorheuristic} as the minimizer of the loss over $I_B(t)$. The second sum is also negative since, for any $i \in I_B(t+1) \setminus I_B(t)$, the  point $(\g x_i, y_i)$ is within the bounded-error tolerance for the new estimate but saturates the loss for the old one:
\begin{equation}
	\ell_{2,\epsilon} ( y_i -  \g x_i^T \g w_j^{t+1} ) \leq \epsilon^2 = \ell_{2,\epsilon} ( y_i -  \g x_i^T \g w_j^{t} ).
\end{equation} 
Finally, the third sum is zero because for $i\notin I_B(t) \cup I_B(t+1) $, the loss is saturated for both estimates: 
\begin{equation}
	\ell_{2,\epsilon} ( y_i -  \g x_i^T \g w_j^{t+1} ) = \ell_{2,\epsilon} ( y_i -  \g x_i^T \g w_j^{t} ) = \epsilon^2.
\end{equation}
Therefore, $J_{\rm BE}^2(\g w_j^{t+1}) \leq J_{\rm BE}^2(\g w_j^t)$, and the cost function value can only decrease.
\end{proof}

To initialize the first upper bound in the global optimization Algorithm~\ref{alg:bb}, we use 100 runs of the algorithm above with random initializations.

\subsection{Lower bounds}

Following the path of Sect.~\ref{sec:switched}, we derive a number of lower bounds of increasing tightness. 

\begin{lemma}\label{lem:pwebe}
Given a box $B = [\g u,\g v]\subset \R^d$, with the notations of Sect.~\ref{sec:switched} and in particular~\eqref{eq:defeiLU},
\begin{equation}
	\underline{J}(B) = \sum_{i\in I_j} \min \left\{ \left(e_i^U(B) \right)_+^2 + \left(e_i^L(B)\right)_-^2  ,\ \epsilon^2 \right\}
\end{equation}
is a lower bound on $\min_{\g w_j \in B} J_{\rm BE}^2(\g w_j)$.
\end{lemma}
\begin{proof}
We proceed as in the beginning of the proof of Lemma~\ref{lem:pwe}: since for any $\g a \in B$, 
\begin{align}
	\min \left\{ e_i^2(\g a), \epsilon^2\right\} 
		&\geq \min_{\g w_j\in B}\min\left\{e_i^2(\g w_j), \epsilon^2\right\}
		= \min\left\{\min_{\g w_j\in B}e_i^2(\g w_j), \epsilon^2\right\},
\end{align}
we have 
\begin{equation}\label{eq:lem:pwebe:2}
	\min_{\g w_j \in B} \sum_{i\in I_j} \min\{e_i^2(\g w_j),\epsilon^2\} \geq \sum_{i\in I_j} \min \left\{ \min_{\g w_j \in B} e_i^2(\g w_j) , \epsilon^2 \right\}.
\end{equation}
Then, noticing that the sum in the left-hand side is $J_{\rm BE}^2(\g w_j)$ and introducing~\eqref{eq:mineij} in the right-hand side completes the proof.
\end{proof}

Let us define the index sets 
\begin{align}
	&I_1(B) = \left\{i \in I_j: \max_{\g w\in[\g u,\g v]} e_i^2 ( \g w) \leq \epsilon^2 \right\},\\ 
	&I_2(B) = \left\{i \in I_j : \min_{\g w\in[\g u,\g v]} e_i^2 ( \g w) > \epsilon^2 \right\}
\end{align}
and 
\begin{equation}
	I_0(B) = I_j \setminus (I_1(B)\cup I_2(B)).
\end{equation}
These sets can be easily determined by using~\eqref{eq:maxeij} for $I_1(B)$ and~\eqref{eq:mineij} for $I_2(B)$. Then, we obtain the following lower bound.
\begin{lemma}\label{lem:lbbe}
Given a box $B = [\g u,\g v]\subset \R^d$ and the notations above, 
\begin{align}
	\underline{J}(B)=&\sum_{i\in I_0(B)} \min \left\{ \left(e_i^U(B)\right)_+^2 + \left( e_i^L(B)\right)_-^2  ,\ \epsilon^2 \right\} + \min_{\g w_j\in B} \sum_{i\in I_1(B)} e_i^2 (\g w_j)  + |I_2(B)| \epsilon^2
\end{align}
is a lower bound on $\min_{\g w_j \in B} J_{\rm BE}^2(\g w_j)$.
\end{lemma}
\begin{proof}
Since each data index $i\in I_j$ is exactly in one and only one of the sets $I_0(B)$, $I_1(B)$ and $I_2(B)$, the cost function can be decomposed as 
\begin{equation}
	J_{\rm BE}^2(\g w_j) = \sum_{k=0}^2 \sum_{i\in I_k(B)} \min\left\{e_i^2(\g w_j),\ \epsilon\right\}.
\end{equation}
By definition of the index sets, this yields, for any $\g w_j\in B$, 
\begin{align}
	J_{\rm BE}^2(\g w_j)=\sum_{i\in I_0(B)}\min\left\{e_i^2(\g w_j),\epsilon\right\} +\sum_{i\in I_1(B)} e_i^2(\g w_j) + |I_2(B)| \epsilon^2.
\end{align}
We can use the fact that Lemma~\ref{lem:pwebe} holds similarly for a sum over any index set instead of $I_j$, and in particular for $I_0(B)$, to lower bound the first term. Then, the result follows since the minimum of the sum of three terms is larger than or equal to the sum of the minimum of each term.
\end{proof}

As for the switching regression case of Section~\ref{sec:switched}, the only demanding task in the lower bound of Lemma~\ref{lem:lbbe} is to solve a box-constrained least squares problem over the points with index in $I_1(B)$. 

Using similar index sets, we also obtain a version of Lemma~\ref{lem:lbbe} for the $\ell_{0,\epsilon}$ loss function.
\begin{lemma}\label{lem:lbbe0}
Given a box $B = [\g u,\g v]\subset \R^d$ and the notations above, $\underline{J}(B) = |I_2(B)|$ is a lower bound on $\min_{\g w_j \in B} J_{\rm BE}^0(\g w_j)$.
\end{lemma}

\subsection{Convergence}

As for the switching regression case studied in Sect.~\ref{sec:switchingcvg}, convergence is obtained, for both the $\ell_{2,\epsilon}$ and the $\ell_{0,\epsilon}$ loss functions, from the tightness of the bounds, under the following assumptions. 
\begin{assumption}\label{ass:globalpositivebe}
For $p=2$ (respectively, $p=0$), the global optimum of Problem~\eqref{eq:Jbe} (resp. Problem~\eqref{eq:Jbe0}) is strictly positive:
\begin{equation}
	J^* = \min_{\g w_j\in B_{\rm init}} J_{\rm BE}^p(\g w_j)	> 0.
\end{equation}
\end{assumption}
\begin{assumption}\label{ass:upperboundsbe}
For any $p\in\{0,2\}$, upper bounds in a box $B=[\g u,\g v]$ are computed as $\overline{J}(B) = J_{\rm BE}^p(\g u)$ or such that $\overline{J}(B) \leq J_{\rm BE}^p(\g u)$. 
\end{assumption}

\begin{theorem}\label{thm:boundederrorcvg}
Under Assumptions~\ref{ass:globalpositivebe}--\ref{ass:upperboundsbe}, the branch-and-bound algorithm described above to minimize $J_{\rm BE}^p$ with $p\in\{0,2\}$ and lower bounds $\underline{J}(B)$ computed as in Lemma~\ref{lem:lbbe} for $p=2$ or Lemma~\ref{lem:lbbe0} for $p=0$ converges in a finite number of iterations for any $TOL>0$.
\end{theorem}
\begin{proof}
The proof is similar to the one of Theorem~\ref{thm:converge}. Its adaptation requires only to show that the lower bounds converge towards the cost function value at the box base point $\g u$ as the iteration counter $t$ tends to the infinity: $\underline{J}(B) \xrightarrow{t\to\infty} J_{\rm BE}^p(\g u)$. As in Theorem~\ref{thm:converge}, this is a consequence of the splitting rule which guarantees that the remaining boxes shrink towards a single point: $B\xrightarrow{t\to\infty} \{\g u\}$. Then, for $p=0$, this directly leads to
\begin{equation}
	\underline{J}(B)=|I_2(B)| \xrightarrow{t\to\infty} \left|\{i\in I_j : e_i^2(\g u) > \epsilon^2\}\right| = J_{\rm BE}^0(\g u).
\end{equation}
For $p=2$, $B\xrightarrow{t\to\infty} \{\g u\}$ yields, for all data index $i\in I_j$, the convergence of $L_i(B)$ and $U_i(B)$ towards zero and thus $e_i^L(B)\xrightarrow{t\to\infty} e_i(\g u)$ and $e_i^U(B)\xrightarrow{t\to\infty} e_i(\g u)$, which imply
\begin{align}
	\underline{J}(B) \xrightarrow{t\to\infty}& \sum_{i\in I_0(B)}\min\left\{e^2_i(\g u) , \epsilon^2\right\} +\min_{\g w_j\in B}\sum_{i\in I_1(B)}e_i^2 (\g w_j)   + |I_2(B)| \epsilon^2\nonumber \\
	\xrightarrow{t\to\infty} &  \sum_{i\in I_0(B)} \min  \left\{e^2_i(\g u) , \epsilon^2\right\} +\sum_{i\in I_1(B)} e_i^2 (\g u)  + |I_2(B)| \epsilon^2 \nonumber\\
	\xrightarrow{t\to\infty} & \ J_{\rm BE}^2(\g u).
\end{align}
Then, we conclude as in Theorem~\ref{thm:converge} using Assumptions~\ref{ass:globalpositivebe}--\ref{ass:upperboundsbe}. 
\end{proof}

\section{Numerical experiments}
\label{sec:exp}

Four sets of experiments are performed to validate the proposed algorithms: for switching regression in Sect.~\ref{sec:expswitched}, for bounded-error identification in Sect.~\ref{sec:expbounded}, for robust estimation in Sect.~\ref{sec:exprobust} and for exact recovery  under sparse noise in Sect.~\ref{sec:expsparse}. 
In all experiments, the initial box bounds on all variables are set to $[-10,10]$. The tolerance on the relative optimality gap is set to $TOL = 0.001$.
The accuracy is measured in terms of the normalized parametric mean squared error, NMSE $= \sum_{j=1}^n \|\g \theta_j - \g w_j\|_2^2 / \|\g \theta_j\|_2^2$, where $\g \theta_j$ is the $j$th parameter vector of the true system, and the classification error rate (CE) equal to the fraction of data points for which the mode $q_i$ is incorrectly estimated. 
The computing times refer to Matlab implementations running on a laptop with a 3Ghz i7-dual core processor. To set the accuracy reference, we use an {\em oracle} based on independent least squares estimations with knowledge of the true classification of the data points (the mode $q_i$ or the inlier/outlier categorization).

\subsection{Switching linear regression}
\label{sec:expswitched}

We first evaluate the global optimization approach to switching regression proposed in Sect.~\ref{sec:switched} with respect to its average computing time as a function of the number of modes $n$, the dimension $d$ and the number of data $N$. For each problem size, we report the average and standard deviation of the computing time over 10 trials, in which the regression vectors $\g x_i$ and the true parameter vectors $\g \theta_j$ are randomly drawn from a uniform distribution in $[-5,5]^d$. For regression vectors, we slightly alterate this in order to avoid the presence of vectors close to the origin (which yield data points that are consistent with all linear models and do not bring information for the estimation). The outputs are then generated with $y_i = \g \theta_{q_i}^T\g x_i + \xi_i$, where the mode $q_i$ is uniformly drawn in $\Q$ and $\xi_i$ is a centered Gaussian noise of standard deviation $\sigma_\xi = 0.1$.  

The results in Table~\ref{tab:switched} show that switching regression problems with up to 10\,000 points in dimension 5 can be solved in about one minute on a standard laptop. But, as expected, the computing time quickly increases with the dimension and the number of modes. Yet, these results support the claim that the complexity of the proposed approach remains reasonably low with respect to the number of data. In particular, Table~\ref{tab:switched} suggests an even less than linear complexity in $N$, indicating that the number of data does not critically influences the number of iterations and mostly affects the linear algebra and convex optimization operations. This would explain why the observed complexity in $N$ is sub-linear instead of linear: when $N$ increases, such operations benefit from standard parallel processing features that do not offer significant speed-ups for small $N$. 
Note that given the NP-hardness of the problem, the high complexity with respect to $n$ and $d$ appears hardly overcomable by any global optimization approach. 

Regarding the accuracy of the approach, both the parametric (NMSE) and classification (CE) errors reported in Table~\ref{tab:switched} are very low (NMSE $<10^{-4}$ and CE $<3\%$ in all cases) and comparable to the ones obtained by the oracle (not reproduced here). In particular, the few remaining classification errors are inherently due to the fact that noisy data points generated with one mode can be better approximated by another one. 
\begin{table}
\centering
\caption{Average and standard deviation of the computing time, the NMSE and the classification error rate (CE) for the global optimization of a switching linear model with $n$ modes for different dimensions $d$ and number of data $N$. \label{tab:switched}}
\begin{tabular}{c|c|r|r|r|r}\hline
	$n$ & $d$ & $N$ & Time (s) & NMSE $\times 10^{4}$ & CE ($\%$)\\\hline
 2 	& 2 & 500  & $0.1 \pm 0.1$ & $0.252\pm 0.419$ & $1.2\pm 0.4$\\	
 	&   & 1\,000  & $0.2 \pm 0.1$ & $0.092\pm 0.079$ & $1.2\pm 0.2$\\	
	&   & 10\,000  & $0.6 \pm 0.4$ & $0.018\pm 0.022$ & $1.6\pm 0.7$\\\cline{2-6}
	& 3 & 500  & $0.8 \pm 1.3$ & $0.101\pm 0.058$ & $0.9\pm 0.5$ \\	
	&	& 1\,000 & $0.6 \pm 0.3$ & $0.335\pm 0.862$ & $0.8\pm 0.5$\\
	&   & 10\,000 & $2.3\pm 1.5$ & $0.009\pm 0.008$ & $1.0\pm 0.4$\\\cline{2-6}
	& 4 & 500 & $4.1\pm 2.9$ & $0.128\pm 0.063$ & $1.0\pm 0.6$\\
	&   & 1\,000 & $5.8\pm 6.9$ & $0.064\pm 0.065$ & $0.8\pm 0.4$ \\
	&   & 10\,000 & $11.9\pm 11.9$ & $0.005\pm 0.002$ & $0.7\pm 0.1$\\\cline{2-6} 
	& 5 & 500 & $24.0\pm 20.7$ & $0.145\pm 0.107$ & $0.7\pm 0.5$\\
	&   & 1\,000 & $35.3\pm 29.9$ & $0.078\pm 0.009$ & $0.6\pm 0.3$ \\
	&   & 10\,000 & $66.7\pm 20.1$ & $0.009\pm 0.005$ & $0.6\pm 0.1$\\\hline		
3 	& 2 & 500  & $1.3 \pm 0.9$ & $0.899\pm 0.824$ & $2.3\pm 0.6$\\	
 	&   & 1\,000  & $1.8 \pm 1.4$ & $0.184\pm 0.126$& $2.3\pm 0.8$ \\	
	&   & 10\,000  & $3.8 \pm 2.2$ & $0.036\pm 0.045$ & $2.2\pm 0.3$\\\cline{2-6}
	& 3 & 500  & $22.8 \pm 23.2$ & $0.289\pm 0.177$ & $1.8\pm 0.6$ \\	
	&	& 1\,000 & $50.7 \pm 45.3$ & $0.143\pm 0.088$ & $1.9\pm 0.8$ \\
	&   & 10\,000 & $72.4\pm 38.8$ & $0.026\pm 0.017$ & $1.8\pm 0.4$\\\cline{2-6}
	& 4 & 500 & $783\pm 626$ & $0.267\pm 0.131$ & $1.5\pm 0.2$\\
	&   & 1\,000 & $1404\pm 977$ & $0.098\pm 0.015$ & $1.4\pm 0.5$ \\
	&   & 10\,000 & $2061\pm 1239$ & $0.015\pm 0.005$ & $1.5\pm 0.1$\\\hline
\end{tabular}
\end{table}

\paragraph*{Switched system identification example} Consider the benchmark example from~\cite{Vidal08}, where the aim is to identify the dynamical system arbitrarily switching between $n=2$ modes as
\begin{equation}
	y_i =\begin{cases}
		-0.9 y_{i-1} + u_i + \xi_i,&\mbox{if } q_i = 1\\
		0.7 y_{i-1} - u_i + \xi_i,&\mbox{if } q_i = 2 
	\end{cases}
\end{equation}
from $N=1000$ data points with $\g x_i=[y_{i-1}, u_i]^T$ and $\xi_i$ a centered Gaussian noise of standard deviation $\sigma_\xi = 0.2$. Over 100 trials with random input ($u_i$) and noise sequences, the proposed algorithm obtains an average NMSE of $1.1547\times10^{-4}$ very close to the one of the oracle equal to $1.1381\times 10^{-4}$ and about 4 times smaller than the $4.8356\times 10^{-4}$ reported in \cite{Vidal08}. 
In this experiment, the average computing time was only $0.43\pm0.05$ seconds and the dynamical nature of the data did not seem to alter the efficiency of the algorithm.

\subsection{Bounded-error identification of switched dynamical systems}
\label{sec:expbounded}

Consider now the example in \cite{Bako11} where the aim is to estimate the system arbitrarily switching between $n=3$ modes as
\begin{equation}\label{eq:systembako}
	y_i=\begin{cases}
		 -0.4 y_{i-1} + 0.25 y_{i-2} - 0.15 u_i + 0.08 u_{i-1} + \xi_i ,& \mbox{if } q_i=1\\
		 1.55 y_{i-1} -0.58 y_{i-2} - 2.1 u_i + 0.96 u_{i-1} + \xi_i ,& \mbox{if } q_i = 2\\
		 y_{i-1} - 0.24 y_{i-2} - 0.65 u_i + 0.3 u_{i-1} + \xi_i ,& \mbox{if } q_i = 3
	\end{cases}
\end{equation}
from $N=300$ data points  with $\g x_i=[y_{i-1}, y_{i-2}, u_i, u_{i-1}]^T$ and a centered Gaussian noise $\xi_i$ of standard deviation $\sigma_\xi$ such that the signal-to-noise ratio is about 30 dB. 
Applying the iterative bounded-error approach depicted in Sect.~\ref{sec:boundederror} with the $\ell_{0,\epsilon}$ and $\ell_{2,\epsilon}$ losses for $\epsilon=1.5\sigma_\xi$, we obtain rather accurate estimates of the parameters, as reported in Table~\ref{tab:paramestimates}. In this setting, the overall computing time for the multiple global optimizations of~\eqref{eq:Jbe0} or~\eqref{eq:Jbe} remains reasonable: a few seconds for the $\ell_{0,\epsilon}$-loss and close to one minute for the $\ell_{2,\epsilon}$-loss. 
In addition, the low complexity with respect to the number of data of the proposed approach is here also verified: in a similar experiment with ten times more data ($N=3000$), the computing time was only multiplied by two for the $\ell_{2,\epsilon}$ loss and grew linearly for the $\ell_{0,\epsilon}$ loss. 

Of course, on this example taken from \cite{Bako11}, the sparse optimization method developed in \cite{Bako11} and based on $\ell_1$-minimization can also achieve a good accuracy in a shorter amount of time. However, extending the system~\eqref{eq:systembako} with 2 more modes while maintaining a uniform distribution for $q_i$ yields a much more difficult problem for this method which then fails to return relevant estimates. 
Indeed, in this case, the dominant mode generates only about $20\%$ of the data points and the optimal error vector becomes far from sparse/compressible, hence breaking the desired behavior of convex relaxations. In such a case, it is critical to consider the global optimization of the nonconvex problem. Applying the proposed approach, we obtain accurate estimates of the five parameter vectors in reasonable time (see the two last rows of Table~\ref{tab:paramestimates}). 

\begin{table}
\centering
\caption{NMSE and computing time (seconds) for the identification of the system~\eqref{eq:systembako} with different loss functions.\label{tab:paramestimates}}
\begin{tabular}{cc|lr|lr|lr}\hline
& &  \multicolumn{2}{c|}{$\ell_{0,\epsilon}$-loss} & \multicolumn{2}{c|}{$\ell_{2,\epsilon}$-loss} & \multicolumn{2}{c}{$\ell_1$-loss \cite{Bako11}} \\
$n$ & $N$ & NMSE & Time & NMSE & Time & NMSE & Time \\\hline
3 & 300 & 0.0030 & 4.5 & 0.0018 & 67 & 0.0502 & 0.2 \\ 
	& 3000 & 0.00046 & 35.0 & 0.00003 & 123 & 0.00003 & 3.4 \\\hline  
5 & 300 & 0.0069 & 14.3 & 0.0052 & 417.2 &  2.8128 & 0.2\\  
  & 3000 & 0.0029 & 95.3 & 0.00025 & 926.3 & 2.0917 & 5.6\\\hline
\end{tabular}
\end{table}

\subsection{Bounded-error estimation in the presence of outliers} 
\label{sec:exprobust}

We now evaluate the bounded-error approach for robust estimation in the presence of outliers. In this case, the problems~\eqref{eq:Jbe} or \eqref{eq:Jbe0} are solved only once to estimate a single model from the maximal number of points that can be considered as inliers. We compare the proposed algorithms with the standard $\ell_1$-minimization in  a setting similar to the one in \cite{Bako16}: an increasing fraction $r$ of 500 data points in dimension 4 are corrupted by outliers $\zeta_i$ drawn from a Gaussian distribution with mean 100 and standard deviation $1000$: $y_i = \g \theta^T\g x_i + \xi_i + \zeta_i$. The results are reported in Fig.~\ref{fig:robust} (top) for $\epsilon=1.5\sigma_{\xi}$ and $\sigma_\xi=0.1$. 

The proposed global minimizations of the $\ell_{0,\epsilon}$ and $\ell_{2,\epsilon}$ losses perform similarly well at rejecting outliers and yielding accurate estimates with an error close to the one of the oracle. On the other hand, the error of the $\ell_1$-minimization quickly increases with more than $60\%$ of outliers. Regarding the computing time, these results are obtained in a few seconds for the $\ell_{0,\epsilon}$ loss and about 30 seconds for the $\ell_{2,\epsilon}$. This is slower than the $\ell_1$-minimization heuristic (which takes less than a second) but still reasonable given the gain in accuracy.

In addition, the $\ell_1$-based method is known to break down when estimating affine models with data corrupted by more than $50\%$ of positive gross errors ($\zeta_i>0$) \cite{Bako16}. Indeed, the bottom plot of Fig.~\ref{fig:robust} shows that, when $\g x_i$ includes a constant component in order to implement an affine model and $\zeta_i$ is replaced by $|\zeta_i|$, the error of this approach quickly increases with the number of outliers. 
At the opposite, the proposed global optimizations are not affected by such adversarial conditions and the error remains comparable to the one of the oracle.

\begin{figure}
\centering
\includegraphics[width=0.45\linewidth]{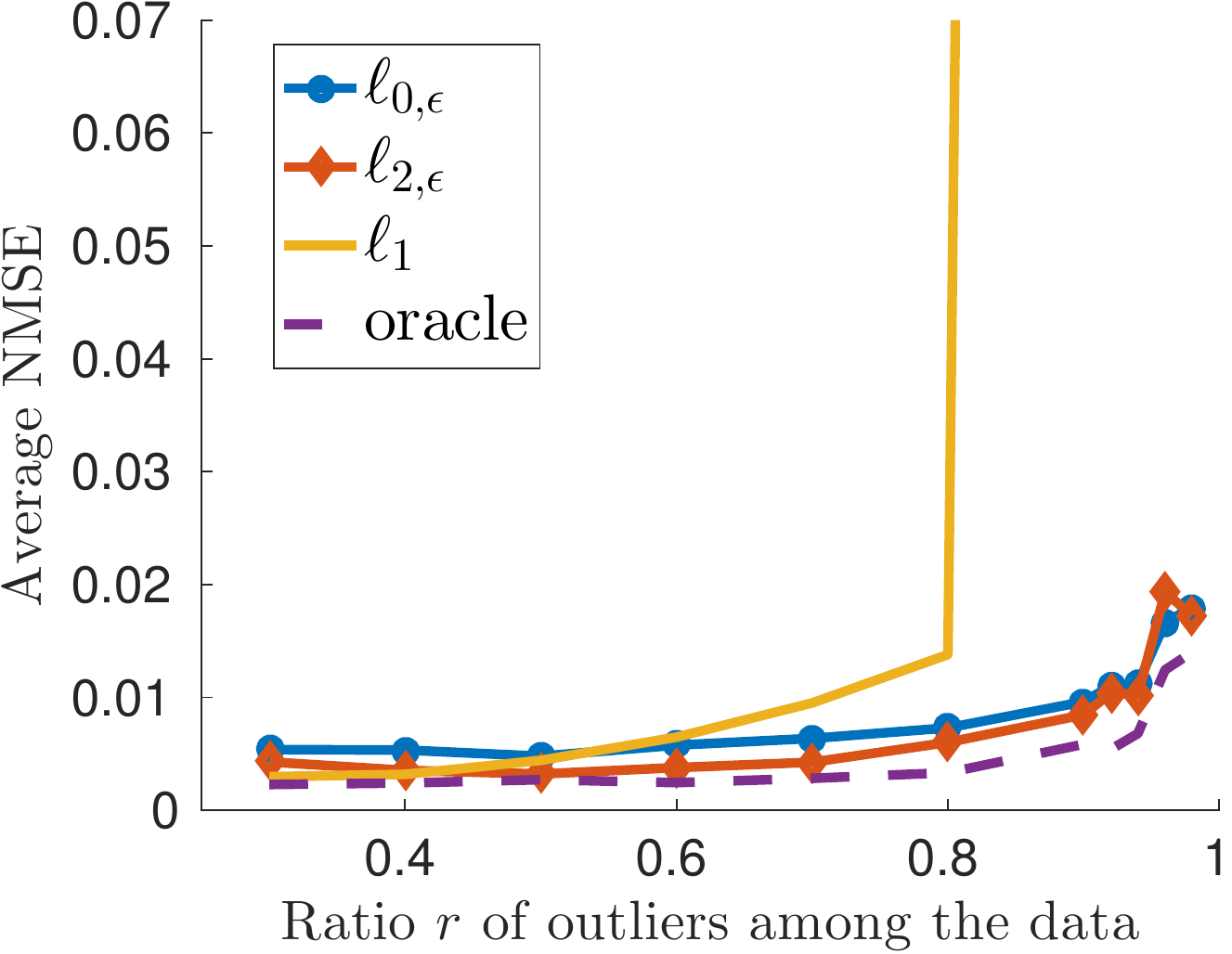}
\includegraphics[width=0.45\linewidth]{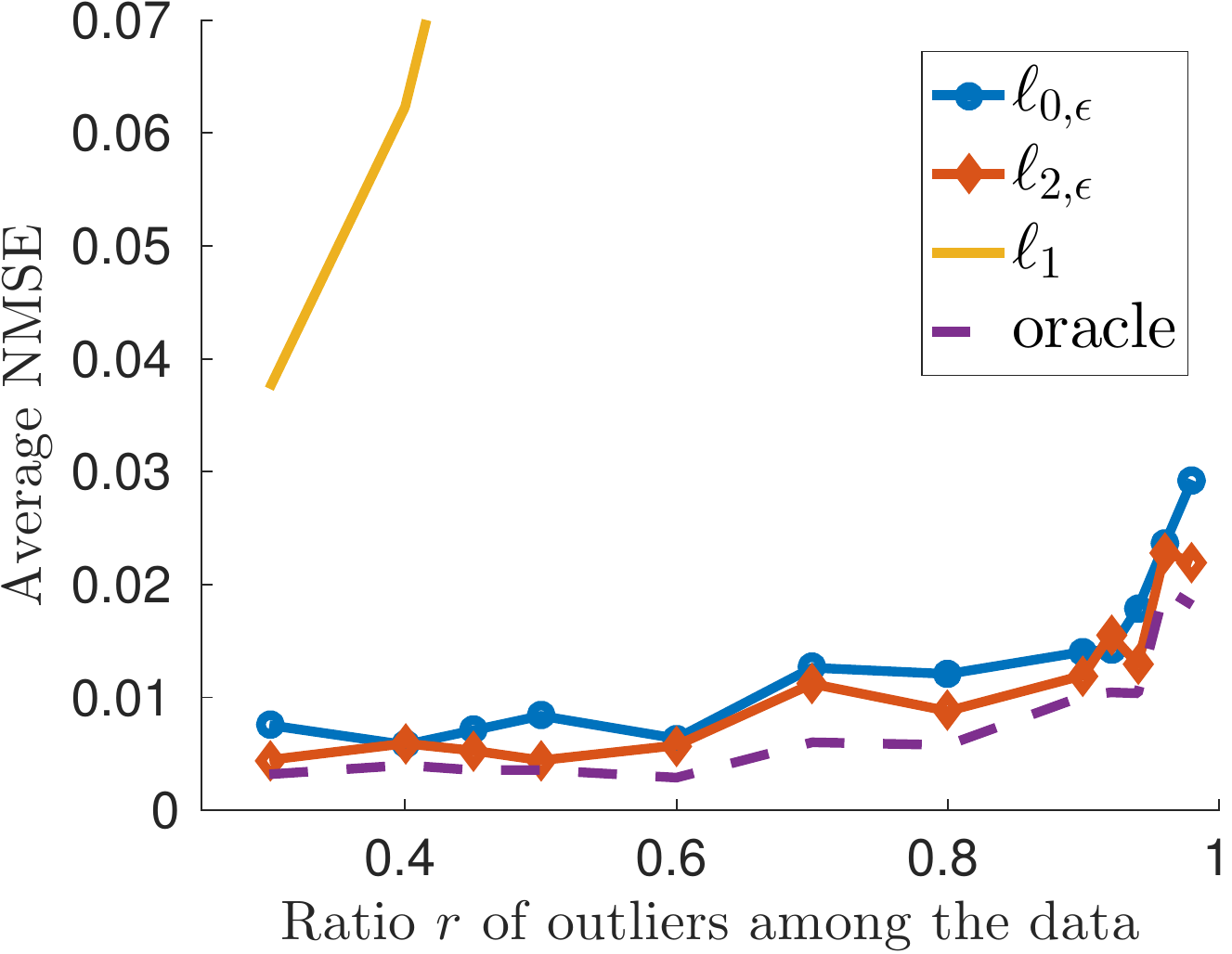}
\caption{Average NMSE over 10 trials vs the ratio of outliers for a linear model (left) and an affine model (right).\label{fig:robust}}
\end{figure}

\subsection{Exact recovery with sparse and gross measurement errors}
\label{sec:expsparse}

Most of the analysis of convex relaxations in \cite{Bako16} actually deals with exact recovery in the case where the data is noiseless ($\xi_i=0$) except for sparse and gross errors. The proposed bounded-error algorithms can also be applied in this setting, e.g., with $\epsilon = 10^{-6}$, and evaluated on the basis of the probabililty of exact recovery, estimated as the fraction of successful trials\footnote{In practice, exact recovery is said to occur when $\|\g w - \g \theta\|_2 < 10^{-6}$.}.   
We study this probability as a function of the ratio $r\in[0.3,0.99]$ of outliers in the data for experiments in a setting similar to Sect.~\ref{sec:exprobust} except that $\sigma_\xi = 0$.\footnote{Note that with $N=500$, a ratio $r=0.99$ of outliers leaves 5 uncorrupted data points, which should be enough to exactly recover $\g \theta\in\R^d$ with $d=4$.} 
In particular, we report in Table~\ref{tab:exactrecovery} the maximal ratio $r$ leading to exact recovery over all trials, i.e., for which the probability estimate is one. This experiment is conducted in two scenarios: with a linear model and outliers $\zeta_i$ or with an affine model ($\g x_i$ includes a constant component) and positive outliers $|\zeta_i|$.

For both scenarios, the results show that the global optimization of nonconvex cost functions provides a larger range of cases for exact recovery than the standard convex relaxation. In addition, this range is unaffected by the linear or affine nature of the model and the constant sign of the outliers, whereas these factors can seriously alter the performance of the $\ell_1$-minimization. In all these trials, the proposed algorithms performed global optimization in less than one second for $r\leq 90\%$ and less than 10 seconds for $r\leq 98\%$. For $r=99\%$, we stopped the algorithm after 2 minutes without obtaining the solution.  

\begin{table}
\centering
\caption{Exact recovery with sparse and gross measurement errors occurring only at $rN$ points. The percentages are the maximal values of $r$ for which the methods achieve exact recovery over all trials.\label{tab:exactrecovery}}
\begin{tabular}{l|c|c|c}\hline
Loss function & $\ell_{0,\epsilon}$ & $\ell_{2,\epsilon}$ & $\ell_{1}$ \\\hline
Linear case & $98\%$ & $98\%$ & $80\%$ \\
Affine case with positive outliers & $98\%$ & $98\%$ & $43\%$ \\\hline
\end{tabular}
\end{table}

\section{Conclusions}
\label{sec:conclusion}

The paper introduced new branch-and-bound algorithms to solve several nonconvex optimization problems relevant to hybrid system identification and robust estimation. Compared with most of the literature, a specificity of the proposed approach is that it offers global optimality guarantees that are independent of the data. 
In addition, by focusing on the continuous variables the approach remains scalable with respect to the number of data. Indeed, 
switching regression problems with thousands of points could be solved in seconds with global optimality certificates for the first time.
 
However, the worst-case complexity remains exponential in the dimension of the data and the number of modes. Given the NP-hardness of the considered problems, such a worst-case complexity can be expected for any global optimization approach. In this regard, an interesting open issue concerns the characterization of the (typically lower) theoretical average time complexity, for instance under a given probability distribution of the data or of the noise.

\end{document}